\documentclass{amsart}
\usepackage{amsmath}%
\usepackage{amsthm}
\usepackage{amsxtra}%
\usepackage{amsfonts}%
\usepackage{amssymb}%
\usepackage[foot]{amsaddr}
\usepackage{bm}%
\usepackage{caption}
\usepackage{tablefootnote}
\usepackage{url}
\usepackage{tabularx}
\usepackage{url}
\usepackage{hyperref}
\usepackage{graphicx}
\usepackage{cite}
\hypersetup{colorlinks,breaklinks,
             linkcolor=blue,urlcolor=blue,
             anchorcolor=blue,citecolor=blue}
\usepackage{xcolor}
\usepackage{cleveref}
\usepackage{float}
\usepackage{microtype}
\usepackage{subfigure}
\usepackage{booktabs} % for professional tables
\usepackage{makecell}
\usepackage{url}
\usepackage{color}
\usepackage{graphicx}
\usepackage{float}
\usepackage{multirow}
\usepackage{siunitx}
\usepackage[margin=1in]{geometry}

\newcommand{\norm}[1]{\Vert#1\Vert}

\DeclareMathOperator{\grad}{\nabla}

\DeclareMathOperator{\Lp}{Lip}
\DeclareMathOperator*{\argmin}{{arg\,min}}

\newcommand{\bq}{\begin{equation}}
\newcommand{\eq}{\end{equation}}

\newcommand{\DD}{{\mathcal{D}}}
\newcommand{\D}{{\mathcal{D}}}
\newcommand{\LL}{L}
\newcommand{\R}{\mathbb{R}}
\newcommand{\F}{\mathcal{F}}
\newcommand{\EE}{\mathbb{E}}

\newcommand{\bO}{\mathcal{O}}

\newcommand{\loss}{{\ell}}

\newcommand{\E}{\mathbb{E}}
\newcommand{\eps}{\varepsilon}
\renewcommand{\P}{\mathbb{P}}

\newcommand{\M}{\mathcal{M}}

\newcommand{\Reg}{\mathcal R}

\theoremstyle{plain}
\newtheorem{theorem}{Theorem}[section]

\newtheorem{corollary}[theorem]{Corollary}
\newtheorem{lemma}[theorem]{Lemma}
\newtheorem{proposition}[theorem]{Proposition}

\theoremstyle{definition}
\newtheorem{definition}[theorem]{Definition}

\newtheorem{remark}[theorem]{Remark}

\numberwithin{equation}{section}

\title[Lipschitz regularized Deep Neural Networks]{Lipschitz regularized Deep Neural Networks generalize and are adversarially robust}
\author{Chris Finlay, Jeff Calder, Bilal Abbasi, \and Adam Oberman}

\address[Finlay, Abbasi, Oberman]{Department of Mathematics and Statistics, McGill University}
\address[Calder]{School of Mathematics, University of Minnesota}
\email{\{adam.oberman,christopher.finlay\}@mcgill.ca, jcalder@umn.edu, bilal.abbasi.ba@gmail.com}

\thanks{
Research supported by:  
AFOSR FA9550-18-1-0167 (A.O.).  NSF-DMS 1713691 (J.C.).
Google gift (B.A.)
}
\begin{document}
\begin{abstract} 
In this work we study input gradient regularization of deep neural networks, and demonstrate that such regularization leads to generalization proofs and improved adversarial robustness.   The proof of generalization does not overcome the curse of dimensionality, but it is independent of the number of layers in the networks.  The adversarial robustness regularization combines adversarial training, which we show to be equivalent to Total Variation regularization, with Lipschitz regularization.  We demonstrate empirically that the regularized models are more robust, and that gradient norms of images can be used for attack detection. 
 \end{abstract}

\maketitle

\section{Introduction}
In this work we propose a modified loss function designed to improve adversarial robustness and allow for a proof of generalization.  Our theoretical contribution is a proof that the regularized model generalizes.  In fact, we prove more: we prove that models converge.  While the convergence rates do not overcome the curse of dimensionality, they are independent of the number of layers of the network.  Empirically, we demonstrate that regularization improves adversarial robustness.  We also use the gradient norm of the loss of a model to detect adversarial perturbed images. 

We show that input gradient regularization improves adversarial robustness, and in particular, that Lipschitz regularization allows us to prove generalization bounds in a nonparametric  setting. The regularized model exhibits a trade-off between the regularization (Lipschitz constant) and the accuracy (expected loss) of a model. We also show that the regularization parameter can be chosen so that the loss of accuracy can controlled.  Our results, if applied to deep neural networks, give a generalization rate independent of the number of layers in the network.   This is in contrast to existing rates for generalization of deep neural networks, such as \cite{bousquet2002stability,xu2012robustness}, which have exponential complexity in the number of layers.

Machine learning models are vulnerable to adversarial attacks at test time \cite{10.1007/978-3-642-40994-3_25}.  Specifically, convolutional neural networks for image classification are vulnerable to small (imperceptible to the human eye) perturbations of the image which lead to misclassification \cite{szegedy2013intriguing, goodfellow2014explaining}. 

Addressing the problem of adversarial examples is a requirement for deploying these models in applications areas where errors are costly.  
One approach to this problem is to establish robustness guarantees. Weng et al.~\cite{weng2018evaluating} and Hein and Andriushchenko \cite{hein2017formal} propose the Lipschitz constant of the model as a measure of the robustness of the model to adversarial attacks.  Cranko et al.~\cite{cranko_lipschitz_2018} showed that Lipschitz regularization can be viewed as a special case of distributional robustness.  Lipschitz regularization has been implemented in \cite{gouk2018regularisation} and \cite{yoshida2017spectral}. Training models to have a small Lipschitz constant  improves empirical adversarial robustness \cite{tsuzuku_margin_2018,cisse2017parseval}. 

Adversarial training \cite{szegedy2013intriguing, goodfellow2014explaining, madry_2017} is the most effective method to improve adversarial robustness.  
However, to date, all method with improved   adversarial robustness have decreased test accuracy. Tsipras et al.~\cite{tsipras2018robustness} claims that adversarial training inevitably decreases model accuracy, while other work \cite{torkamani2013convex, goodfellow2014explaining, miyato2018virtual} argues that adversarial training should improve the accuracy of a model. 
 
The Lipschitz constant of a model also appears in the study of generalization. Bartlett et al.~\cite{bartlett2017spectrally} proposed the Lipschitz constant of the network as a candidate measure for the Rademacher complexity~\cite {shalev-shwartz_ben-david_2014}.   Data independent upper bounds on the  Lipschitz constant of the model go back to \cite{bartlett1997valid}, while robustness and generalization is studied in \cite{sokolic2017robust}. 

Another challenge is to develop defended models which maintain high accuracy on test data.  Existing defence methods lose accuracy on (non-attacked) test sets.  Tsipras et al.~\cite{tsipras2018robustness} argue that there is a trade-off between accuracy and robustness.  Assessing the effectiveness of defences requires careful testing with multiple attacks \cite{goodfellow_2018}.  Since existing defences may also be vulnerable to new and stronger attacks, Carlini and Wagner \cite{carlini2017adversarial} and Athalye et al.~\cite{athalye2018obfuscated} recommend designing specialized attacks against the model.   As an alternative to attacks, robustness measures can be used  to estimate the vulnerability of a model to
adversarial attack. Weng et al.~\cite{weng2018evaluating} and Hein and Andriushchenko \cite{hein2017formal} propose the Lipschitz constant of the model  as a robustness measure.  Xiao et al.~\cite{xiao2018training} design networks amenable to fast robustness verification.

%This paper is organized as follows. In Section \ref{sec:variational}, we discuss variational problems and gradient regularization in Section \ref{sec:inputgrad} we introduce input gradient regularization. In Section \ref{sec:intro_attacks} we discuss adversarial attacks and training to defend against attacks, and in Section \ref{sec:LipGen}, we use input gradient regularization to prove generalization bounds. Finally, in Section \ref{sec:results} we present results of using input gradient regularization for adversarial training.

\section{Machine learning, variational problems, and gradient regularization}
\label{sec:variational}

In parametric learning problems, models are trained by minimizing the training loss
\begin{equation}\label{eq:trainingloss}
\min_{u\in \F }\sum_{i=1}^n \ell(u(x_i),y_i),
\end{equation}
where $(x_i,y_i)$ is the training data, $\ell$ is the loss, and $\F$ is the parametric hypothesis space. As the hypothesis space $\F$ grows in complexity, the learning problem \eqref{eq:trainingloss} becomes ill-posed, in the sense that many functions $u\in \F$ will give very small training loss, and some will generalize well while others will overfit. In this case, we treat the problem as a nonparametric learning problem, and some form of regularization is required to encourage the learning algorithm to choose a good minimizer $u\in \F$ from among all the choices that yield small training loss. Even though modern deep learning is technically a parametric learning problem, the high degree of expressibility of deep neural networks effectively renders the problem nonparametric. For such problems, regularization of the loss is necessary to overcome ill-posedness and lead to a well-posed problem, which can be addressed by the theory of inverse problems in the calculus of variations~\cite{dacorogna2007direct}.  This theory plays a fundamental role in mathematical approaches to image processing~\cite{aubert2006mathematical}.   

Gradient regularization is classical in ill-posed problems and leads to a variational problem of the form
\begin{equation}\label{eq:var_reg}
\min_{u \in \F }\{L[u,u_0] + \lambda \Reg[\nabla  u]\},
\end{equation}
where $\Reg$ is the gradient regularizer, $\lambda>0$ is a parameter, and $L[u,u_0]$ is the \emph{fidelity}, which is analogous to the loss in \eqref{eq:trainingloss}, and measures the discrepancy between the image $u$ and the true image $u_0$. Here, the function space $\F$ may encode boundary conditions, such as Dirichlet or Neumann conditions. The classical Tychonov regularization \cite{tikhonov2009solutions} corresponds to $\Reg[u]=\| |\nabla u(x)| \|_{L^2(X)}^2$, while the Total Variation regularization model of~\cite{rudin1992nonlinear} corresponds to  $\Reg[u]= \| |\nabla u(x)| \|_{L^1(X)}$, where $X$ is the domain of $u$.   Each of these regularizations use $L^p$ norms of the gradient.  We recall that {Rademacher's Theorem}~\cite[\S 3.1]{evans2018measure} says that 
\begin{equation}
	\label{LipRad}
\Lp(u) =	\| |\nabla u(x)| \|_{L^\infty(X) },
\end{equation}
when $X$ is convex, where $\Lp(u)$ is the Lipschitz constant of $u:X\to Y$ given by
\begin{equation}\label{eq:Lip}
\Lp(u) = \sup_{\substack{x,y\in X\\ x\neq y}} \frac{\|u(x)-u(y)\|_Y}{\|x-y\|_X}.
\end{equation}
Thus, Lipschitz regularization corresponds to regularization of the $L^\infty$ norm of the gradient $ \| |\nabla u(x)| \|_{L^\infty(X)}$.

Lipschitz regularization appears in image inpainting \cite{bertalmio2000image},
where it is used to fill in missing information from images, as well as in \cite[\S 4.4]{pock2010global} \cite{elion2007image} and \cite{Guillot2009}. It has also been used recently in graph-based semi-supervised learning to propagate label information on graphs~\cite{kyng2015algorithms,calder2017consistency}. Tychonoff regularization of neural networks, which was studied by~\cite{drucker1992improving} is equivalent to training with noise~\cite{bishop1995training} and leads to better generalization. 
%Variational methods are used in semi-supervised learning, such as the graph Laplacian  \cite{10.1007/978-3-540-27819-1_43, zhu2003semi}
Recent work on semi-supervised learning suggests that higher $p$-norms of the  gradient are needed for generalization when the data manifold is not well approximated by the data~\cite{el2016asymptotic, calder2018game, slepcev2017analysis}. 
%The applications of variational inverse problems to learning was studied in~\cite{vito2005learning}.

In the case that $\F=\{u:X\to Y\, : \, u=g \text{ on }\partial X\}$, the variational problem \eqref{eq:var_reg} can be interpreted as a relaxation of the Lipschitz Extension problem
\begin{equation}\label{LipExt}
\min_{\substack{u:X\to Y\\ u=g \text{ on }\partial X}} \Lp(u),
\end{equation}
where the parameter $\lambda^{-1}$ in \eqref{eq:var_reg} is a parameter that replaces the unknown Lagrange multiplier. The problem \eqref{LipExt} has more that one solution.  Two classical results giving explicit solutions in one dimension go back to Kirzbaum and to McShane~\cite{mcshane1934extension}.  However solving \eqref{LipExt} is not practical for large scale problems.  There has be extensive work on the Lipschitz Extension problem, see, \cite{johnson1984extensions}, for example.  More recently, optimal Lipschitz extensions have been studied, with connections to Partial Differential Equations, see~\cite{aronsson2004tour}.  

In general, variational problems can be studied by the direct method in the calculus of variations~\cite{dacorogna2007direct}, which establishes existence of minimizers.  To solve a general convex variational problem, we can discretize the functional to obtain  a finite dimensional convex optimization problem, that can be solved numerically. Alternatively, one can find the first variation, which is a partial differential equation (PDE)~\cite{evansbook}, and then solve the PDE numerically.  Both approaches are discussed in~\cite{aubert2006mathematical}.

\subsection{Robustness and the Lipschitz constant}\label{sec:robustLip}\label{sec:reg_lit}\label{sec:ImpLip}\label{sec:bartlett}
%\subsection*{The Lipschitz constant and robustness guarantees}
In deep learning, the Lipschitz constant of a model appears in the context of
 model robustness \cite{xu2012robustness},   generalization
 \cite{bartlett1997valid,sokolic2017robust}, and Wasserstein
 GANs~\cite{gulrajani2017improved, miyato_spectral_2018, anil2018sorting}.
 
 In the study of adversarial robustness, Weng et al.~\cite{weng2018evaluating} and Hein and Andriushchenko \cite{hein2017formal} showed that the Lipschitz
constant of the model gives a certifiable minimum adversarial distance under
which the model is robust to perturbations. Thus, if the Lipschitz of the model
can be controlled, the model will be robust.
Indeed, training models to have small Lipschitz constant has empirically been shown to
improve adversarial robustness \cite{tsuzuku_margin_2018,cisse2017parseval}. 
By its very definition, the Lipschitz constant determines model robustness
(sensitivity to changes in the data):

However estimating the Lipschitz constant of a deep model can be challenging.  Data independent upper bounds on the  Lipschitz constant of the model go back to \cite{bartlett1997valid}.  These bounds are based on the product of the norm of weight matrices, but the gap in the bound can grow exponentially in the number
of layers, since it  neglects the effects of the activation function.
Alternative methods, as in \cite{weng2018evaluating} can be costly. Estimating the Lipschitz constant by the maximum gradient norm over the training data can be a gross underestimate, since \eqref{LipRad} indicates that the gradient norm should be evaluated over the entire domain. 

\subsection{Input gradient regularization}
\label{sec:inputgrad}

We now describe our proposed input gradient regularization for training deep neural networks. We have two objectives to achieve with gradient regularization: (i) enable proofs of generalization for deep learning, and (ii) improve adversarial robustness. 

Let $X\subset \R^d$ denote the data space, and let $Y$ denote the label space. Let $\DD_n=\{x_1,\dots,x_n\}\subset X$ denote a sample of data of size $n$, and let $y_1,\dots,y_n\in Y$ denote the corresponding labels. Let $\ell:Y\times Y\to \R$ denote a loss function. We define the \emph{input gradient regularized loss} to be 
\begin{equation}\label{main_equation}
J_n[u]:=\underbrace{\frac{1}{n}\sum_{i=1}^n\ell(u(x_i),y_i)}_{\text{Training loss}} + \underbrace{\eps \| \nabla u \|_{L^1(X)}}_{\text{TV regularization}} + \underbrace{\lambda \Lp(u),}_{\text{Lipschitz regularization}}
\end{equation}
where $\eps,\lambda>0$, and $\Lp(u)$ is the Lipschitz constant of $u$ defined in \eqref{eq:Lip}. The model includes a penalty for the average and maximum values of the norm of the gradient.  The average corresponds to the Total Variation regularization term, while the maximum  term corresponds to Lipschitz regularization. 

For the purpose of generalization, we need only consider the Lipschitz regularization. That is, we set $\epsilon = 0$ in \eqref{main_equation} to obtain the problem
\begin{equation}
\label{Main_Problem_Data_Intro}
	\min_{u: X \to Y}\left\{ \frac 1 n \sum_{i=1}^n\ell(u(x_i), y_i) + \lambda\Lp(u)\right\}.
\end{equation}
Letting $u_n$ be any minimizer in \eqref{Main_Problem_Data_Intro}, we study the limit of $u_n$ as $n\to \infty$.  Our main results, Theorems~\ref{thm:gen3} and~\ref{thm:gen2},  show the that minimizers of $J^n$  converge,  as the number of training points $n$ tends to $\infty$, to minimizers of a limiting functional $J$. This result is stronger than generalization, which only requires that the empirical loss converges to the expected loss. We state and prove our main generalization results in Section \ref{sec:LipGen}.  See Figure~\ref{fig:solution} for an illustration of Lipschitz regularization~\eqref{eq:var_reg} of a synthetic classification problem.

For defense against adversarial attacks, we focus on the total variation regularization by setting $\eps>0$ and $\lambda=0$ in \eqref{main_equation}. To make the problem tractable numerically, we numerically approximate the regularization term on the training dataset as
\begin{equation}\label{eq:TVnum}
\| \nabla u \|_{L^1(X)} \approx \frac{1}{n}\sum_{i=1}^n |\nabla_x \ell(u(x_i),y_i)|,
\end{equation}
which is less computationally complex compared to computing the full gradient $\nabla u$. When discussing adversarial training, we will often write $\ell(u(x),y)$ as $\ell(x)$, since the dependence on the labels  $y$ is not important.
%By the chain rule
%\[|\nabla_x \ell(u(x_i),y_i)| \leq |\nabla_u\ell(u(x_i),y_i)\cdot \nabla u(x_i)| \leq C|\nabla u(x_i)|.\]
Thus, our input gradient regularization for adversarial defenses can be written in the form 
\begin{equation}\label{main_equation_AT}
  %\underset{(x,y)\sim\D}
  {\EE_{x\sim \DD_n}}\left[  \loss\left(x\right )\right]  + \varepsilon 
  \underbrace{\EE_{x\sim \DD_n} \left[\|\nabla_x
\loss(x)\|  \right]}_{\text{Total Variation}} 
\end{equation}
where $\EE_{x\sim \DD_n}$ denotes the empirical expectation with respect to the dataset $\DD_n$.  We show in Section \ref{sec:intro_attacks} that the total variation term in \eqref{main_equation_AT} can be interpreted as the regularization which arises from one-step adversarial training.

Implementing \eqref{main_equation_AT} is achieved using a modification of one step adversarial training. At each iteration, the norm gradient of the model loss is computed for each image in the batch.  The Total Variation term is the average
over the batch, and the Lipschitz term is the maximum over the batch.  
It is possible to compute $\nabla_x \ell(x)$ using backpropagation, however this means training is slower, because
it requires a second pass of backpropagation for the model parameters, $w$. 
A more efficient method is  to 
 accurately estimate $\|\nabla_x \ell(x)\|$ via a finite difference approximation \cite{iserles2009first}:
\[
  \|\nabla_x \ell(x) \| \approx \frac{1}{h} \left( \ell(x+hd)-\ell(x) \right )
\]
where $d$ is the attack direction vector, and $h$ is small (we use $h = $ \num{1e-3}).

%We consider the following input gradient regularized loss
%\begin{equation}
%	\label{main_equation}
%		 {\EE}\left[  \loss(x)\right  ] 
%	 + \varepsilon \underbrace{{\EE} \left [ \|\nabla_x \loss(x)\|_*  \right]}_{\text{TV regularization}}   + \lambda \underbrace{\max \|\nabla_x \loss(x) \|_*}_{\text{Lipschitz regularization}}.
%\end{equation}

\begin{figure}[htb]
\centering
\includegraphics[width=2.5in,height=2in]{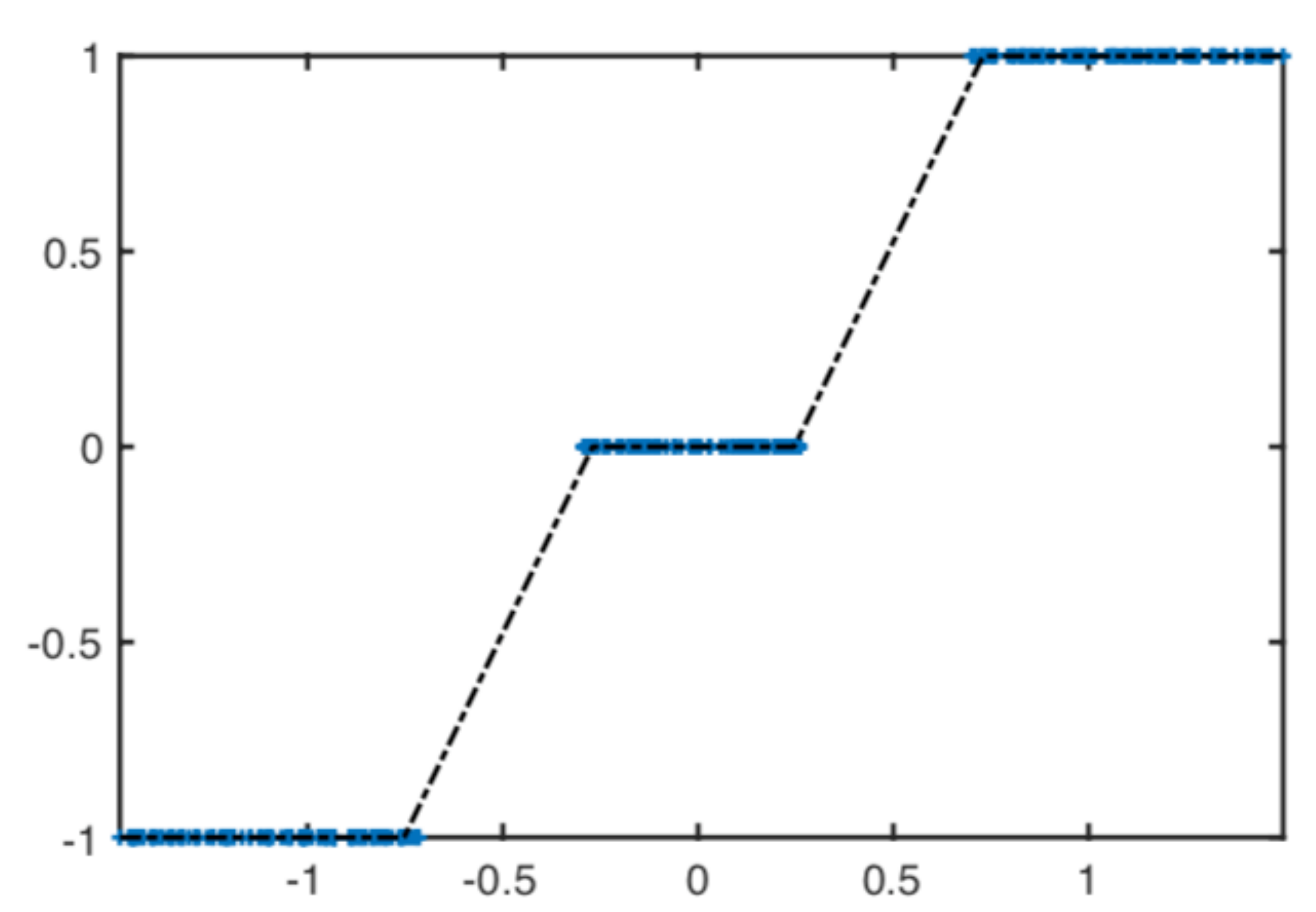}
\hspace{0.75cm} 
\includegraphics[width=2.5in,height=2in]{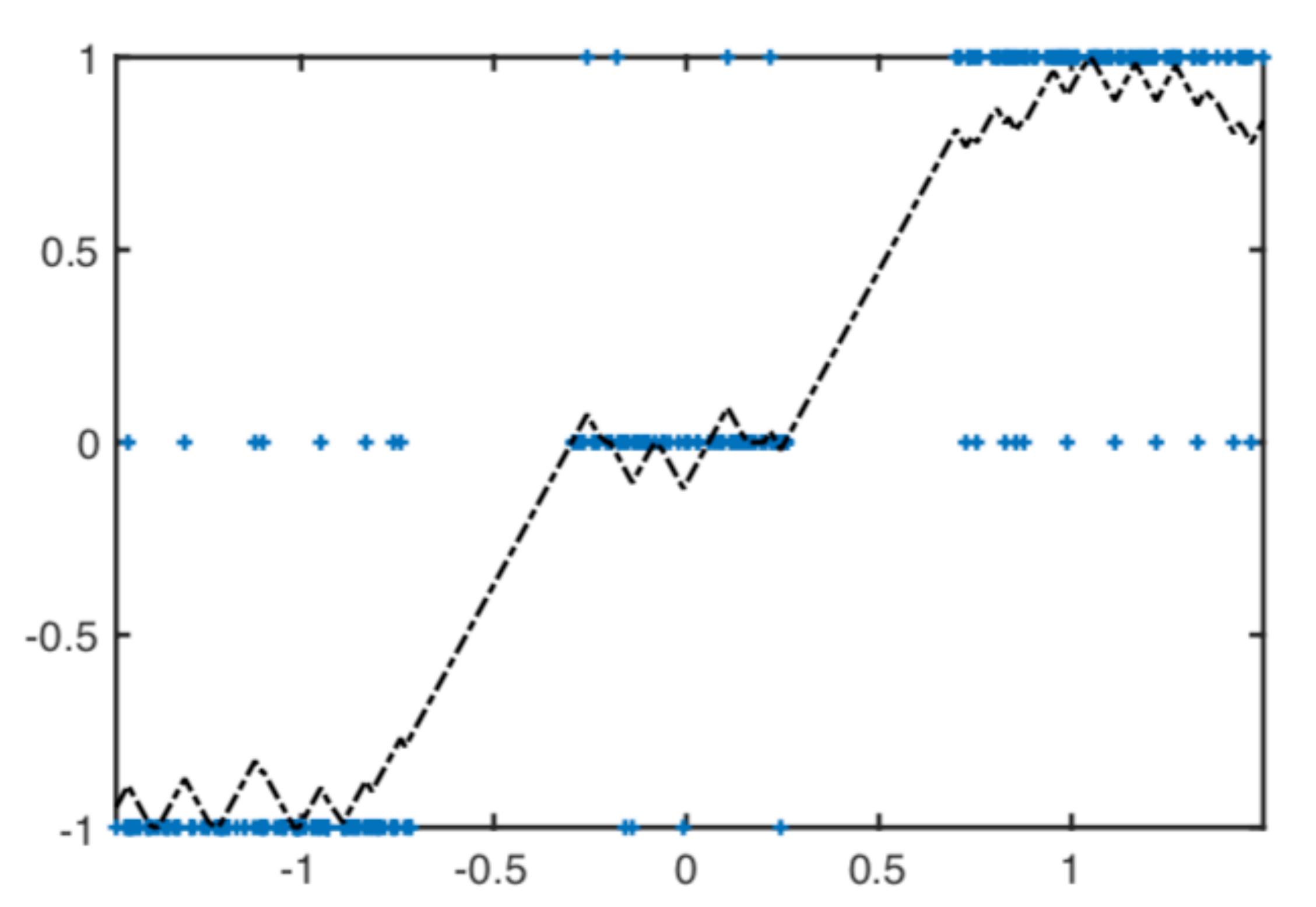}
\caption{Synthetic labelled data and Lipschitz regularized solution $u$.
%The data is represented in blue on the $x$-axis, the labels are the points $\{-1,0,1\}$ on the $y$-axis.  
Left: The solution value matches the labels perfectly on a large portion of the data set. Right: 10\% of the data is corrupted by incorrect labels; the regularized solution corrects the errors.}
\label{fig:solution}
\end{figure}

\section{Adversarial attacks and defences}
\label{sec:intro_attacks}

Adversarial attacks seek to find a small perturbation of an image vector which leads to a
misclassification by the model. Since undefended models are vulnerable to very small norm attacks, the goal of adversarial defence methods is to force successful attack vectors to be larger in size.  To allow gradients to be used, the classification attack is usually replaced by an attack on the loss.

 The earliest and most successful defence to adversarial perturbations is
 adversarial training \cite{szegedy2013intriguing, goodfellow2014explaining,
 tramer_ensemble_2017, madry_2017, kannan2018adversarial}. 
Adversarial training  corresponds to teaching the network to correctly identify
adversarially perturbed images by replacing training images $x$ with attacked images $x'$.
The attacked images used by Madry et al.~\cite{madry_2017} arise from I-FGSM.  
However determining the best choice of attack vectors is model dependent, and requires choosing 
a number of hyperparameters (for example $\varepsilon$ and the
number of iterations).  Current models trade off model robustness and model accuracy; \cite{tsipras2018robustness} claims this cannot be avoided.

%Our notation is as follows.  Write $y = f(x,w)$ for a model which takes input vectors $x$ to label probabilities, with parameters $w$.  Let $\mathcal{L}(x,y)$ be the loss (usually the cross-entropy loss).  Write $\loss(x) = \mathcal{L}(f(x,w),y)$ for the model loss.  In writing $\loss(x)$, we suppress the dependence on the labels, $y$, and the parameters, $w$, because they will not be needed in what follows.

Focusing on the model loss, $\loss(x)$, defined in Section \ref{sec:inputgrad},   adversarial perturbations can be modelled as an optimization problem.
For a given adversarial distance, denoted $\varepsilon$, the optimal loss attack on the image vector, $x$, is the solution of
\begin{equation}
\label{ideal_attack}
\max_{ \|x'-x\| \leq \varepsilon} \loss(x').
\end{equation}
The solution, $x'$ is the perturbed image vector within a given distance of $x$, which maximally increases the loss.

The Fast Gradient Sign Method (FGSM) \cite{goodfellow2014explaining} arises when
attacks are measured in the $\infty$-norm.  It corresponds to a one step attack in
the direction $d_1$ given by the signed gradient
\begin{align} \label{eq:l1_dual_vector}	
(d_1)_i = \frac{ \nabla \ell(x)_i }{| \nabla \ell(x)_i|}.
\end{align}
The attack direction \eqref{eq:l1_dual_vector} arises from linearization of the objective in \eqref{ideal_attack}, which leads to 
\begin{equation}\label{norm_infty_dual}
	\max_{ \|d\|_\infty \leq 1}  d \cdot \nabla \loss(x) 
\end{equation} 
By inspection, the minimizer is the signed gradient vector,
\eqref{eq:l1_dual_vector}, and the optimal value %of \eqref{norm_infty_dual} 
is $\| \nabla \loss(x)\|_1$.  When the  $2$-norm is used in
\eqref{norm_infty_dual}, the optimal value is $\| \nabla
\loss(x)\|_2$ and the optimal direction is the normalized gradient
\begin{align} \label{eq:l2_dual_vector}
  \dfrac{ \nabla \ell(x) }{\| \nabla \ell(x) \|_2}.
\end{align}
More generally, when a generic norm is used in \eqref{norm_infty_dual}, the
maximum of the linearized objective is the dual norm~\cite[A.1.6]{boyd2004convex}.

Iterative versions of these attacks lead to respectively Iterative FGSM (I-FGSM)
\cite{kurakin2016adversarial}, %or projected gradient descent in the $\infty$-norm, 
and projected gradient descent in the 2-norm.  
These attacks require gradients of
the loss. However, other attacks use only model gradients (such as the Carlini-Wagner
attack \cite{carlini_towards_2016}), while others use no gradients at all, but only the classification
decision of the model (for example the Boundary
attack \cite{brendel2018decisionbased}).

\subsection{Image vulnerability}

Given the model, $f$, and a correctly classified image vector, $x$, how can we predict
when it is vulnerable to attack?     This question has already been
addressed by examining the \emph{model} Lipschitz constant, which gives a bound on the worst-case adversarial distance \cite{hein2017formal,weng2018evaluating}.  However the Lipschitz constant does not distinguish whether one image $x$ is more vulnerable to attack than another. 

We argue that we can compare the vulnerability of different images using \emph{the size of the model loss gradient} at the image. To do so, we make the simplifying assumption that the attack is a one-step attack.  This argument leads to useful results, which we show later are effective even for general attacks. In the case of a one step attack, this argument is reasonable: consider the extreme case of a zero gradient: then the one-step attack does not perturb the image at all.  In the other extreme case, of a very large gradient, a small one step attack will change the loss by a large amount. 

Given an adversarially perturbed training image $x'$, write $x' = x+\varepsilon d$, where $d$ is a unit vector, $d$, and $\varepsilon$ is small. 
Using a Taylor expansion \cite{iserles2009first} of the loss, we obtain 
\begin{equation}\label{Taylor_Loss}
\loss(x')  = \loss(x) + \varepsilon d \cdot \nabla \loss(x) + \mathcal
O(\varepsilon^2).
\end{equation}

Substitute $d$ for either the signed gradient \eqref{eq:l1_dual_vector} or the normalized
gradient \eqref{eq:l2_dual_vector} in \eqref{Taylor_Loss} to obtain 
\begin{equation} 
	\label{bound_1}
\loss(x+\varepsilon d) = \loss(x) + \varepsilon \|\nabla \loss (x)\|_* + \bO(\varepsilon^2)
\end{equation}
 where $*$ is dual to the norm measuring adversarial perturbations.
Thus \emph{images with larger loss gradients are more vulnerable to gradient attacks}. Moreover,
the type of attack vector used (either signed gradient or normalized gradient) determines
the correct norm in which to measure the vulnerability.

We can go a little further with \eqref{bound_1}, which measures the vulnerability of an image to attack, by asking what is the size perturbation needed to misclassify an image. Set  $\loss_0$ to be the smallest value of the loss for which the image is
misclassified. Then, for small values of $\varepsilon$, using \eqref{bound_1},  we can estimate the
minimum adversarial distance by 
\begin{equation}
	\label{robustx_defn}
	\varepsilon(x) = \frac{\loss_0 - \loss(x)}{\|\nabla\loss(x)\|}, \quad \text{ when }\loss(x) \leq \loss_0 
\end{equation}
and $\varepsilon(x) = 0$ when $x$ is already misclassified. 
We remark that for the cross entropy loss, $\ell_0$ is bounded above by
\begin{equation}\label{l0_estimate}
	\loss_0 \leq \log( \text{number of labels} ).
\end{equation}

\section{Lipschitz regularization and generalization}
\label{sec:LipGen}

\subsection{Generic regularizers}

Before studying Lipschitz regularization, we step back and consider properties of general regularizers. Let $\mathcal H$ be a collection of function $h:X\to Y$ containing the constant functions.

\begin{definition}\label{def:reg}
 We say $\Reg :\mathcal H \to \R$  is a \emph{regularizer} if (i) $\Reg(h) \geq 0$ and (ii) $\Reg(h) = 0$ if and only if $h$ is constant function.
	 For a given loss $L$, define the regularized learning problem
	\begin{equation}\label{RLM}\tag{RLM}
	\min_{f} \{ L(f) + \lambda \Reg(f)\}.
	\end{equation}
The problem \eqref{RLM} corresponds to \emph{regularized loss minimization} (RLM) instead of \emph{expected loss minimization} (ELM).
Define for $\lambda \geq 0$
\begin{equation}
	\label{ulambda_defn}
	f^\lambda \in \argmin \{ L(f) + \lambda \Reg(f)\}.
\end{equation}
\end{definition}
\begin{remark}
   The loss $L$ in Definition \ref{def:reg} would normally be the training loss on a training set, as in Eq.~\eqref{main_equation}. In this case, the regularization functional depends on the entire function $f$, and not just the restriction of $f$ to the training data. In practice, it should be evaluated on a validation set, which leads to additional $O(1/\sigma{m})$ error in evaluating it, which is independent of dimension.
\end{remark}

Immediately, we can show that we have some good properties of the regularized functional.  In the next lemma, we show that for small values of $\lambda$ the minimizers, $f^\lambda$ of \eqref{RLM} increase the empirical loss at most linearly in $\lambda$.  We also show that for large values of $\lambda$ the value of the regularization term goes to zero as $1/\lambda$. 
\begin{definition}
	The \emph{realizable} case is the case where $f^* \in \mathcal H$, where $f^*$ is the true label function. In particular, we can always achieve zero empirical loss. Define $\bar f$ to be the average of $f^*$ and write $C_L = L(\bar f)$. 
%	
%Alternative: 
%Define 
%$\bar u$ to be a function which minimizes the loss, among all functions which satisfy $\Reg(h) = 0$ 
%\[
%\bar u \in \argmin   \left \{ L_S(h) \mid \Reg(h) = 0 \right \}, 
%\qquad
%C_L = L_S(\bar u)
%\]
%in particular by the assumption on the regularizer, $\bar u$ is a constant function. 
\end{definition}

\begin{lemma}
	Define $f^\lambda$ as in \eqref{ulambda_defn}, and assume the loss is nonnegative.  Then for any $\lambda>0$
\begin{equation}\label{J_bounds}
\left\{	\begin{aligned}
	L(f^\lambda) &\leq  L(f^0) + \lambda \Reg(f^0) \\
	\Reg(f^\lambda) & \leq \frac 1 \lambda C_L
\end{aligned}\right.
\end{equation}
Note in the realizable case, $L(f^0) = 0$. 
\end{lemma}

\begin{proof}
Since $\Reg(f^\lambda)\geq 0$ we have
\[L(f^\lambda) \leq L(f^\lambda) + \lambda \Reg(f^\lambda) \leq L(f^0) + \lambda\Reg(f^0),\]
which establishes the first inequality in \eqref{J_bounds}. For the other inequality, we use the nonnegativity of the loss to find that
\[\lambda\Reg(f^\lambda) \leq L(f^\lambda) + \lambda \Reg(f^\lambda) \leq L(\bar{f}) = C_L,\]
which completes the proof.
%(In addition can show that $\Reg(f^\lambda) \leq \Reg(f_S^0)$ as well). 
\end{proof}

The previous result gives a bound on the empirical loss of a minimizer of \eqref{RLM} for a generic regularizer $\Reg$, but it does not say anything about the expected loss.   However it does say that we can regularize without changing the empirical loss too much.  On the other hand, for large values of the regularizer, we expect that $f^\lambda$ is going to a constant, which suggests that the \emph{generalization gap} is small.   However, not every regularizer is a good regularizer, since we could have the trivial regularizer
\begin{equation}
	\Reg^0(h) =
	\begin{cases}
		0 & h \text{ is constant } \\
		1 & \text{ otherwise } 
	\end{cases}
\end{equation}
which does nothing and then forces the solution to a constant.

\subsection{Lipschitz regularization leads to generalization}

We now turn to study generalization properties of Lipschitz regularization. Generalization in machine learning refers to bounds on the generalization gap
\[\E_{x\sim \rho} \left [ \ell(u_n(x), y(x))\right ] - \E_{x\sim \DD_n}\left [ \ell(u_n(x), y(x))\right ]\]
where $u_n:X\to Y$ denotes the function learned  from the training data $\DD_n$ by solving \ref{Main_Problem_Data_Intro}. We prove something stronger; namely, we prove that the solutions $u_n$ converge almost surely as $n\to \infty$ to the solution of a continuum variational problem of the form
\begin{equation}\label{Main_Problem}%\tag*{$(J)$}
\min_{u: X \to Y}  J[u] := \E_{x\sim \rho} \left [ \ell(u(x), y(x))\right ]  +  \lambda \Lp(u).
\end{equation}
Theorems~\ref{thm:gen2} and~\ref{thm:gen3} below show uniform convergence, and convergence with a rate in the strongly convex case, both in the limit as the number of training points $n$ tends to $\infty$. While our rates are dimension dependent, they do not depend on the number of layers in the deep neural network.

We now describe our assumptions on the loss and data distribution. We always make the assumption that  $\ell(y_1,y_2) \geq 0$    for all  $y_1, y_2$ with equality if and only if $y_1 = y_2$.  We say $\ell(y_1,y_2)$ is strictly convex, if it is strictly convex as a function of $y_1$ for all values of $y_2$.  We say $\ell$ is \emph{strongly convex} with parameter $\theta>0$, if 
\begin{equation}\label{eq:sconvex}
\ell(ty_1 + (1-t)y_2,y_0) + \tfrac{\theta}{2}t(1-t)\|y_1-y_2\|_{Y}^2  \leq t\ell(y_1,y_0) + (1-t)\ell(y_2,y_0)
\end{equation}
holds for all $y_0,y_1,y_2\in Y$ and $0 \leq t \leq 1$.
\begin{remark}
\label{rem:crossentropy}
The standard cross-entropy loss is not strongly convex or Lipschitz continuous.  In order to apply our results, we can use the \emph{regularized cross entropy loss} with parameter $\epsilon >0$,
\begin{equation}\label{eq:regce}
\ell^{KL}_\eps(y,z) =-  \sum_{i=1}^D {(z_i + \epsilon)} \log\left( \frac{y_i + \epsilon}{z_i + \epsilon} \right).
\end{equation}
which is Lipschitz continuous and strongly convex on the probability simplex. 	
\end{remark}

We make the standard manifold assumption \cite{ssl}, and assume the data distribution $\rho$ is a probability density supported on a compact, smooth, $m$-dimensional manifold $\M$ embedded in $X=[0,1]^d$, where $m \ll d$. We denote the probability density again by $\rho:\M\to [0,\infty)$. Hence, the data $\D_n$ is a sequence $x_1,\dots,x_n$ of \emph{i.i.d.}~random variables on $\M$ with probability density $\rho$. For $\lambda>0$ we let $u^\lambda$ denote any minimizer of \eqref{Main_Problem}, that is
\[
u^\lambda \in \argmin_{u:X\to Y} J[u].
\]
While $u^\lambda$ is not unique, Lemma \ref{lem:unique} below shows that $u^\lambda$ is unique on the data manifold when the loss $\ell$ is strictly convex. Thus, the choice of minimizer $u^\lambda$ is unimportant, since Lemma \ref{lem:unique} guarantees the minimizer is unique on the data manifold, and all minimizers of \eqref{Main_Problem} have the same Lipschitz constant. 
%Proposition~\ref{prop_loss} below gives a guarantee that $u^n$ and $u^\lambda$ are robust to adversarial perturbations. 
\begin{lemma}\label{lem:unique} Suppose the loss function is strictly convex in $y_1$. If $u,v\in W^{1,\infty}(X;Y)$ are two minimizers of \eqref{Main_Problem} and $\inf_{\M}\rho>0$ then $u=v$ on $\M$.
\end{lemma}
We recall that $W^{1,\infty}(X;Y)$ is the space of Lipschitz mappings from $X$ to $Y$.   
\begin{proof}
Let $w = (u+v)/2$. Then by convexity of $J$
\begin{align}\label{eq:convexity}
J[w] \leq  \tfrac{1}{2}J[u] + \tfrac{1}{2}J[v] = \min_{u}J[u].
\end{align}
Therefore, $w$ is a minimizer of $J$ and so we have equality above, which yields
\[\int_\M \left[ \tfrac{1}{2}\ell\left(u,y \right) + \tfrac{1}{2}\ell\left(v,y \right)\right]\rho\, dVol(x) = \int_\M \ell\left( \tfrac{1}{2}u + \tfrac{1}{2}v,y \right)\rho\, dVol(x).\]
Since $\ell$ is strictly convex in its first argument, it follows that $u=v$ on $\M$.
\end{proof}

We now state our main generalization results.
\begin{theorem}\label{thm:gen2}
Suppose that $\inf_\M \rho > 0$ and $\ell:Y\times Y\to \R$ is Lipschitz and strictly convex. Then with probability one
\begin{equation}\label{eq:convM}
u_n \longrightarrow u^\lambda \ \ \text{uniformly on }\M \text{ as }n\to \infty,
\end{equation}
where $u_n$ is any sequence of minimizers of \eqref{Main_Problem_Data_Intro}. Furthermore, every uniformly convergent subsequence of $u_n$ converges on $X$ to a minimizer of \eqref{Main_Problem}.
\end{theorem}

If we also assume strong convexity of the loss, then we get a convergence rate. 
\begin{theorem}\label{thm:gen3}
Suppose that $\ell:Y\times Y\to \R$ is Lipschitz and strongly convex. Then for any $t>0$, with probability at least $1-2t^{-\frac{m}{m+2}}n^{-(ct - 1)}$ all minimizing sequences $u_n$ of \eqref{Main_Problem_Data_Intro} satisfy
\begin{equation}\label{eq:conv3}
\frac{\theta}{2}\int_\M \|u_n - u^\lambda\|_Y^2\rho\, dVol(x) \leq C\lambda^{-1}J[u^\lambda]\left( \frac{t\log(n)}{n} \right)^{\frac{1}{m+2}}.
\end{equation}
\end{theorem}
In Theorem \ref{thm:gen2} and Theorem \ref{thm:gen3}, the sequence $u_n$ does not, in general, converge on the whole domain $X$. The important point is that the sequence converges on the data manifold $\M$, and solves the variational problem \eqref{Main_Problem} off of the manifold, which provides regularization away from the training data.

Finally, due to the uniform Lipschitzness of the sequence $u_n$, which is established in the proof of Theorem \ref{thm:gen3}, we can upgrade the $L^2$ convergence rate from Theorem \ref{thm:gen3} to an $L^\infty$-rate.
\begin{corollary}\label{cor:Linfinity}
Suppose that $\ell:Y\times Y\to \R$ is Lipschitz and strongly convex. Then for any $t>0$, with probability at least $1-2t^{-\frac{m}{m+2}}n^{-(ct - 1)}$ all minimizing sequences $u_n$ of \eqref{Main_Problem_Data_Intro} satisfy
\begin{equation}\label{eq:interp1}
\|u_n-u^\lambda\|_{L^\infty(\M,Y)}^{m+2} \leq Cm^2L_n^m\lambda^{-1}J[u^\lambda]\left( \frac{t\log(n)}{n} \right)^{\frac{1}{m+2}},
\end{equation}
when $\|u_n-u^\lambda\|_{L^\infty(\M,Y)}\leq L_nr$, and 
\begin{equation}\label{eq:interp2}
\|u_n-u^\lambda\|^2_{L^\infty(\M,Y)}\leq Cm2^mr^{-m}\lambda^{-1}J[u^\lambda]\left( \frac{t\log(n)}{n} \right)^{\frac{1}{m+2}},
\end{equation}
when $\|u_n-u^\lambda\|_{L^\infty(\M,Y)}\geq L_nr$, where $L_n = \Lp(u_n-u^\lambda)$ and $r$ depends only on $\M$.
\end{corollary}
The proof of Corollary \ref{cor:Linfinity} follows from Theorem \ref{thm:gen3} and a standard interpolation inequality, which we recall for completeness in Theorem \ref{thm:interpolation} in \S\ref{sec:interpolation}.

We now turn to the proofs of Theorems \ref{thm:gen2} and \ref{thm:gen3}, which require some additional notation and preliminary results. Throughout this section, we write
\[
\LL[u,\mu] := \E_{x\sim \mu}[\ell(u(x),y(x))] =\int_{X} \ell(u(x),y(x)) d\mu(x) 	
\]
where $\mu$ is a probability measure supported on $X$.  In particular, write   $\rho_n := \frac 1 n \sum \delta_{x_i}$ for the empirical measure corresponding to~${\DD}_n$, so that  $\LL[u,\rho_n]$  is the expected loss on ${\DD}_n$. We also let $C,c>0$ denote positive constants depending only on $\M$, and we assume $C\geq 1$ and $0< c <1$. We follow the analysis tradition of allowing the particular values of $C$ and $c$ to change from line to line, to avoid keeping track of too many constants.  Finally, we Let $H_L(X;Y)$ denote the collection of $L$-Lipschitz functions $w:X\to Y$.

The proofs of of Theorems \ref{thm:gen2} and \ref{thm:gen3} require a discrepancy result.
\begin{lemma}\label{lem:discrepancy}
Suppose that  $\inf_\M \rho > 0$, and $dim(\M) = m$. Then for any $t>0$
\begin{equation}\label{eq:discrepancy}
\sup_{w\in H_L(X;Y)}\left|\frac{1}{n}\sum_{i=1}^n w(x_i) - \int_\M w \rho \, dVol(x)\right|\leq CL\left( \frac{t\log(n)}{n} \right)^{\frac{1}{m+2}}
\end{equation}
holds with probability at least $1-2t^{-\frac{m}{m +2}}n^{-(ct - 1)}$.
\end{lemma}
The proof of Lemma \ref{lem:discrepancy} is given in \S\ref{sec:app_proof}.  The estimate \eqref{eq:discrepancy} is called a discrepancy result \cite{talagrand2006generic,gyorfi2006distribution}, and is a uniform version of concentration inequalities.  The exponent $1/(m+2)$ is not optimal, but affords a very simple proof. It is possible to prove a similar result with the optimal exponent $1/m$ in dimension $m\geq 3$, but the proof is significantly more involved. We refer the reader to \cite{talagrand2006generic} for details.

%In particular  $\ell^{KL}_\eps(y,e_k) = -(1 + \epsilon)\log((y_k + \eps)/(1+\eps))$, which is easily seen to be Lipschitz and strongly convex for any $y$ within the probability simplex. 

%\begin{example}[Cross-entropy]
%In classification, the output of the network is a probability vector on the labels.  Thus $Y = \Delta_D$, the $D$-dimensional probability simplex, and each label is mapped to a basis vector. 
%The cross-entropy loss 
%$\ell^{KL}(y,z) = -\sum_{i=1}^D z_i \log (y_i/z_i)$. 
%For labels, $\ell^{KL}(y,e_k) = - \log( y_k)$. 
%\end{example}
%

%Thus, the regularized cross entropy $\ell^{KL}_{\eps}$ is strongly convex and Lipschitz on the probability simplex. 

We now give the proof of Theorem \ref{thm:gen2}.
\begin{proof}[Proof of Theorem \ref{thm:gen2}]
By Lemma \ref{lem:discrepancy} the event that
\begin{equation}\label{eq:disc}
\lim_{n\to\infty}\sup_{w\in H_L(X;Y)}\left|L[w,\rho_n]- L[w,\rho]\right|=0
\end{equation}
for all Lipschitz constants $L>0$ has probability one. For the rest of the proof we restrict ourselves to this event.

Let $u_n\in W^{1,\infty}(X;Y)$ be a sequence of minimizers of \eqref{Main_Problem_Data_Intro}, and let $u^\lambda\in W^{1,\infty}(X;Y)$ be any minimizer of \eqref{Main_Problem}. Then  since
\[
J^n[u_n]\leq J^n[u^\lambda] 
\] 
we have
\[ L[u_n,\rho_n] + \lambda \Lp(u_n) \leq + L[u^\lambda,\rho_n] + \lambda \Lp(u^\lambda). \]
 Write $L_0=\Lp(u^\lambda)$. Then  
\begin{equation}\label{eq:Lpbound}
\Lp(u_n) \leq L_0  + \lambda^{-1}(L[u^\lambda,\rho_n] - L[u_n,\rho_n]).
\end{equation}
In particular, $\Lp(u_n)$ is uniformly bounded in $n$. Since the label space $Y$ is compact, the sequence $u_n:X\to Y$ is uniformly bounded as well.

By the Arzel\`a-Ascoli Theorem \cite{rudin1976principles} there exists a subsequence $u_{n_j}$ and a function $u\in W^{1,\infty}(X;Y)$ such that $u_{n_j}\to u$ uniformly as $n_j\to \infty$. Note we also have $\Lp(u)\leq \liminf_{j\to\infty}\Lp(u_{n_j})$. Since
\begin{align*}
|L[u_n,\rho_n] - L[u,\rho]|&\leq |L[u_n,\rho_n] - L[u,\rho_n]| + |L[u,\rho_n] - L[u,\rho]|\\
&\leq C\|u_n-u\|_{L^\infty(\M;Y)} + \sup_{w\in H_L(X;Y)}\left|L[w,\rho_n]- L[w,\rho]\right|
\end{align*}
it follows from \eqref{eq:disc} that $L[u_{n_j},\rho_{n_j}] \to L[u,\rho]$ as $j\to \infty$. It also follows from \eqref{eq:disc} that $J^n[u^\lambda] \to J[u^\lambda]$ as $n\to \infty$. Therefore
\begin{align*}
J[u^\lambda]&=\lim_{n\to\infty}J^n[u^\lambda]\\
&\geq \liminf_{n\to\infty}J^n[u_n]\\
&= \liminf_{n\to\infty} L[u_n,\rho_n] + \lambda \Lp(u_n) \\
&= \lim_{n\to\infty} L[u_n,\rho_n] + \liminf_{n\to \infty}\lambda \Lp(u_n) \\
&\geq L[u,\rho] + \lambda \Lp(u) = J[u].
\end{align*}
Therefore, $u$ is a minimizer of $J$. By Lemma \ref{lem:unique}, $u=u^\lambda$ on $\M$, and so $u_{n_j} \to u^\lambda$ uniformly on $\M$ as $j\to \infty$.

Now, suppose that \eqref{eq:convM} does not hold. Then there exists a subsequence $u_{n_j}$ and $\delta>0$ such that
\[\max_{x\in \M}|u_{n_j}(x) - u^\lambda(x)|> \delta\]
for all $j\geq 1$. However, we can apply the argument above to extract a further subsequence of $u_{n_j}$ that converges uniformly on $\M$ to $u^\lambda$, which is a contradiction. This completes the proof.
\end{proof}

The proof of Theorem \ref{thm:gen3} requires some additional preliminary setup.
\begin{proposition}\label{prop:sconvex}
If $\ell$ is strongly convex with parameter $\theta>0$ then
\begin{equation}\label{eq:Jconvex}
J[tu_1 + (1-t)u_2] + \tfrac{\theta}{2}t(1-t)\int_\M \|u_1-u_2\|_Y^2\rho \, dVol(x) \leq tJ[u_1] + (1-t)J[u_2]
\end{equation}
for all $u_1,u_2\in W^{1,\infty}(X;Y)$ and $0 \leq t \leq 1$.
\end{proposition}
\begin{proof}
We compute
\begin{align*}
J[tu_1 + (1-t)u_2]&=\int_\M \ell(tu_1 + (1-t)u_2,y)\rho\, dVol(x) + \lambda   \Lp(tu_1 + (1-t)u_2)\\
&\leq tJ[u_1] + (1-t)J[u_2] - \frac{\theta}{2}t(1-t)\int_\M \|u_1-u_2\|_Y^2\rho \, dVol(x),
\end{align*}
which completes the proof.
\end{proof}
\begin{lemma}\label{lem:mins}
If $u^\lambda\in W^{1,\infty}(X;Y)$ is a minimizer of \eqref{Main_Problem} and $u\in W^{1,\infty}(X;Y)$ then
\begin{equation}\label{eq:minrate}
\frac{\theta}{2}\int_\M \|u - u^\lambda\|_Y^2\rho \, dVol(x)\leq J[u] - J[u^\lambda].
\end{equation}
\end{lemma}
\begin{proof}
We use Proposition \ref{prop:sconvex} with $u_1=u^\lambda$ and $u_2=u$ to obtain
\[J[tu^\lambda + (1-t)u] + \frac{\theta}{2}t(1-t)\int_\M \|u^\lambda-u\|_Y^2\rho \, dVol(x) \leq tJ[u^\lambda] + (1-t)J[u].\]
Since $J[tu^\lambda + (1-t)u] \geq J[u^\lambda]$ 
\[J[u^\lambda] + \frac{\theta}{2}t(1-t)\int_\M \|u^\lambda-u\|_Y^2\rho \, dVol(x) \leq tJ[u^\lambda] + (1-t)J[u],\]
and so
\[\frac{\theta}{2}t\int_\M \|u^\lambda-u\|_Y^2\rho \, dVol(x) \leq J[u] - J[u^\lambda].\]
Setting $t=1$ completes the proof.
\end{proof}

We now give the proof of Theorem \ref{thm:gen3}.
\begin{proof}[Proof of Theorem \ref{thm:gen3}]
By Lemma \ref{lem:discrepancy}
\begin{equation}\label{eq:disc2}
\sup_{w\in H_{L_0}(X;Y)}\left|L[w,\rho_n]- L[w,\rho]\right|\leq CL_0\left( \frac{t\log(n)}{n} \right)^{\frac{1}{m+2}}
\end{equation}
holds with probability at least $1-2t^{-\frac{m}{m+2}}n^{-(ct - 1)}$ for any $t>0$. Let us assume for the rest of the proof that \eqref{eq:disc2} holds.

By \eqref{eq:Lpbound} in the proof of Theorem \ref{thm:gen2}, we have 
\[\Lp(u_n)\leq L_0 + D_n,\]
where $D_n=\lambda^{-1}(L[u^\lambda,\rho_n] - L[u_n,\rho_n])$. By \eqref{eq:disc2} we have
\[D_n \leq \lambda^{-1}L[u^\lambda,\rho] + CL_0\left( \frac{t\log(n)}{n} \right)^{\frac{1}{m+2}}.\]
%The line above could be made stronger if we assume an upper bound $\ell(x,y)-\ell(z,y) \leq C|x-z|^2$, but this may be too strong to assume. This would change the $J[u^\lambda]$ back to $L_0$ below.
Therefore
\begin{equation}\label{eq:un}
\left|J^n[u^\lambda]- J[u^\lambda]\right|,\left|J^n[u_n]- J[u_n]\right|\leq C(L_0+\lambda^{-1}L[u^\lambda,\rho])\left( \frac{t\log(n)}{n} \right)^{\frac{1}{m+2}},
\end{equation}
and since $L_0+\lambda^{-1}L[u^\lambda,\rho]=\lambda^{-1}J[u^\lambda]$ we have
\begin{align}
J[u_n] - J[u^\lambda]&=J^n[u_n] - J[u^\lambda] + J[u_n] - J^n[u_n] \leq  C\lambda^{-1}J[u^\lambda]\left( \frac{t\log(n)}{n} \right)^{\frac{1}{m+2}}.
\end{align}
By Lemma \ref{lem:mins} we deduce
\[\frac{\theta}{2}\int_\M \|u_n - u^\lambda\|_Y^2\rho \, dVol(x)\leq C\lambda^{-1}J[u^\lambda]\left( \frac{t\log(n)}{n} \right)^{\frac{1}{m+2}},\]
which completes the proof.
\end{proof}

%\subsection{Proof of Lemma~\ref{lem:2}}
%\label{sec:proofs}
%\label{sec:proofs2}
%
%\begin{proof}[Proof of Lemma~\ref{lem:2}]
%There exists $\eps_\M$ such that for any $0 < \eps \leq \eps_\M$, we can cover $\M$ with $N$ geodesic balls $B_1,B_2,\dots,B_{N}$ of radius $\eps$, where $N\leq C\eps^{-m}$ and $C$ depends only on $\M$ \cite{gyorfi2006distribution}. Let $Z_i$ denote the number of random variables $x_1,\dots,x_n$ falling in $B_i$. Then $Z_i\sim B(n,p_i)$, where $p_i =\int_{B_i}\rho(x)\, dVol(x)$. Since $\rho\geq \theta>0$ and $Vol(B_i)\geq c\eps^m$ we have $p_i\geq c\eps^m$.  Let $A_n$ denote the event that at least one $B_i$ is empty (i.e., $Z_i=0$ for some $i$). Then by the union bound we deduce
%\begin{align}\label{eq:PA}
%\P(A_n) &\leq \sum_{i=1}^{N}\P\left(Z_i =0\right)\\
%&\leq C\eps^{-d}(1-c \eps^m)^{n}\\
%&=C\exp\left( n\log(1-c \eps^m) - \log(\eps^m) \right)\\
%&\leq C\exp\left( -cn \eps^m - \log(\eps^m) \right).
%\end{align}
%Choose $0<\eps \leq \eps_\M$ in the form $n\eps^m = t\log(n)$ with $t \leq n\eps_\M^m/\log(n)$. Then 
%\[\P(A_n) \leq Ct^{-1}\exp\left( -(ct-1)\log(n) \right).\]
%In the event that $A_n$ does not occur, then each $B_i$ has at least one point, and so $|x-\sigma_n(x)|\leq C\eps$ for all $x\in \M$. Therefore
%\[\|\I - \sigma_n\|_{L^\infty(\M;X)}\leq C\eps = C\left(\frac{t\log(n)}{n}\right)^{1/m}\]
%with probability at least $1-Ct^{-1}\exp\left( -(ct-1)\log(n) \right)$. Since $\|\I - \sigma_n\|_{L^\infty(\M;X)} \leq C\sqrt{d}$, the result holds for $t \geq n\eps_\M^m/\log(n)$, albeit with a larger constant $C$.
%\end{proof}
%

\newcolumntype{R}[1]{>{\raggedleft\let\newline\\\arraybackslash\hspace{0pt}}m{#1}}
\newcolumntype{C}[1]{>{\center\let\newline\\\arraybackslash\hspace{0pt}}m{#1}}
%\NewDocumentCommand{\rot}{O{45} O{1em} m}{\makebox[#2][l]{\rotatebox{#1}{#3}}}%

\section{Adversarial robustness results}\label{sec:results}
\begin{figure}
    \centering
    \includegraphics[height=2in]{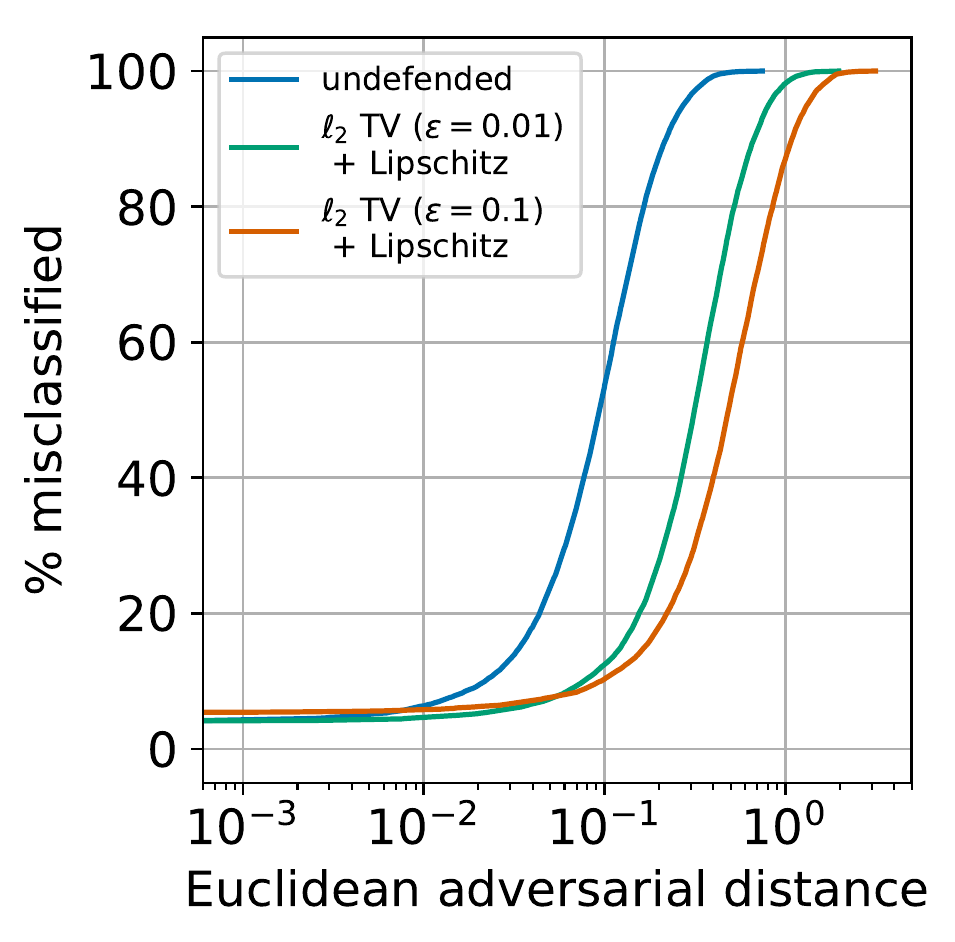}
    \caption{Adversarial robustness results against $\ell_2$ PGD attacks using
    a ResNeXt architecture on the CIFAR-10 test set. The probability of a
  successful attack is plotted against distance in $\ell_2$, for an undefended
network and two defended networks.}
  \label{fig:stability-curves}
\end{figure}

%\defcitealias{madry_2017}{M. et al}
%\defcitealias{qian2018}{Q \& W}

\begin{table*}
  \caption{Comparison with other defence methods on CIFAR-10. Misclassification
    error (percentage) is reported on
    test images.  Each row corresponds to an adversarial attack method. }
  \label{tab:compare_methods}
   \centering
%  \resizebox{0.48\textwidth}{!}
{
  \begin{tabular}{l l l || c|c|c| c| c}
    \multicolumn{3}{c}{}&\multicolumn{4}{c}{Defence method} \\
    \cmidrule(lr){4-8}

    \multirow{3}{*}{Attack}&\multicolumn{2}{c||}{\multirow{3}{*}{details}} &
    \multicolumn{3}{c|}{$\ell_2$ TV + Lipschitz} &
%    \multirow{3}{*}{\citetalias{madry_2017}}&\multirow{3}{*}{\citetalias{qian2018}} \\
    \multirow{3}{*}{Madry}&\multirow{3}{*}{Qian} \\
    & & &  $\varepsilon=0.01$, & $\varepsilon=0.1$, & $\varepsilon=0.1$, & & \\
    & & &  $\lambda=0.1$       & $\lambda=0.1$      & $\lambda=1$       & & \\
    \hline\hline
                          \multirow{2}{*}{test error} &
                          \multicolumn{2}{c||}{undefended model}  &
                          \textbf{4.1}& \textbf{4.1}&\textbf{4.1} & 4.8 & 5.0  \\
                          & \multicolumn{2}{c||}{ defended model}
                          &\textbf{4.1} & 5.4 &   6.0   & 12.8         &
                          22.8          \\ \hline
                          $\ell_2$ PGD & \multirow{3}{*}{distance} & $\|\varepsilon\|_2 = 100/255$                                                          
                                              & 59.8
                                              &37.2 &\textbf{36.1}              & $>90$              & -               \\ 
CW & &$\|\varepsilon\|_2 = 1.5$                & 90.8                        &
91.2    & 84.4                  & -                  & \textbf{79.6}   \\ 
I-FGSM & &$\|\varepsilon\|_\infty ={8}/{255}$  & 98.1                        &
91.6 & 93.7                     & \textbf{54.2}      & -            \\ \hline
\end{tabular}
}
\end{table*}

\begin{table*}
  \caption{Adversarial statistics with ResNeXt-34 (2x32) on CIFAR-10 test data.}
  \centering
\resizebox{\textwidth}{!}{
  \begin{tabular}{ l||  r | r  r  r | r r  | r r  | r  r }
    
    \multicolumn{2}{c}{}& \multicolumn{3}{c}{adversarial distance 2-norm} &
    \multicolumn{2}{c}{adversarial distance $\infty$-norm} &
    \multicolumn{2}{c}{max test statistics} & \multicolumn{2}{c}{mean test statistics}\\
    \cmidrule(lr){3-5}\cmidrule{6-7} \cmidrule(lr){8-9} \cmidrule(l){10-11}
    
     \multirow{2}{*}{defence method} &
    \% Err at& median & median & \multicolumn{1}{c}{\% Err at} & median & \% Err at & 
\multirow{2}{*}{$\norm{\grad \ell}_2$}  &
    \multirow{2}{*}{ $\norm{\grad f}_{2,\infty}$} & \multirow{2}{*}{$\norm{\grad \ell}_2$}  &
    \multirow{2}{*}{ $\norm{\grad f}_{2,\infty}$} \\
     & $\varepsilon=0$& distance &time & $\varepsilon=0.1$ & 
     distance &  
    $\varepsilon=\frac{8}{255}$ & & \\
    \hline
    \hline
    undefended  & 4.07  & 0.09 & 159 & 53.98 & \num{2.7e-3} & 100  & 85.21
    & 13.70 & 1.90 & 0.37 \\
    $\ell_1$ TV (AT, FGSM)                   &3.87      &0.19 & 301       &23.26    &
    \num{5.6e-3} & 92.74  & 35.77 & 6.27 & 1.10 & 0.21 \\
    $\ell_2$ TV ($\varepsilon=0.01$)         & \textbf{3.58}  & 0.30 & 471 & 13.54 & \num{9.0e-3} &
    98.34 & 32.13 & 5.22 & 0.59 & 0.11 \\
    $\ell_2$ TV ($\varepsilon=0.01$) + Lipschitz ($\lambda=0.1$) & 4.13  &0.31 & 473  & 12.52 &\num{9.1e-3} & 98.10   &
     \textbf{4.10} & \textbf{2.14} & \textbf{0.55} & \textbf{0.10} \\
     $\ell_2$ TV ($\varepsilon=0.1$) + Lipschitz ($\lambda=0.1$) & 5.37  & 0.48 &
     659  &
     \textbf{10.31}  & \num{13.7e-3} & \textbf{91.6} & 31.96& 8.93  &1.19   & 0.47 \\
     $\ell_2$ TV ($\varepsilon=0.1$) + Lipschitz ($\lambda=1$) & 5.98  &
     \textbf{0.52} & \textbf{698}  &
     10.95  & \textbf{\num{14.7e-2}} & 93.7 & 18.53 & 4.87  &1.02   & 0.46 \\
    \hline
  \end{tabular}
}
  \label{tab:punch10}
\end{table*}

\begin{table*}
  \caption{Adversarial statistics with ResNeXt-34 (4x32) on CIFAR-100. Training with $\varepsilon = 0.01$ and $\lambda = 0.1$ }
  \centering
\resizebox{\textwidth}{!}{
  \begin{tabular}{ l||  r | r   r r | r r  | r r | r r}
    \multicolumn{2}{c}{}& \multicolumn{3}{c}{adversarial distance 2-norm} &
    \multicolumn{2}{c}{adversarial distance $\infty$-norm}&
    \multicolumn{2}{c}{max test statistics} & \multicolumn{2}{c}{mean test statistics} \\
    \cmidrule(lr){3-5}\cmidrule{6-7} \cmidrule(lr){8-9} \cmidrule(l){10-11} 
    
     \multirow{2}{*}{defence method} &
    \% Err at& median & median & \% Err at & median & \% Err at  &
\multirow{2}{*}{$\norm{\grad \ell}_2$}  &
    \multirow{2}{*}{ $\norm{\grad f}_{2,\infty}$} & \multirow{2}{*}{$\norm{\grad \ell}_2$}  &
    \multirow{2}{*}{ $\norm{\grad f}_{2,\infty}$} \\
    &  $\varepsilon=0$& distance &time& $\varepsilon=0.1$ &  distance &
    $\varepsilon=\frac{1}{255}$ & & \\
    \hline
    \hline
    undefended& \textbf{21.24} & \num{4.7e-2} & 152 &
    74.18  &\num{1.4e-3} & 81.52    & 93.83 & 1.89   & 5.19 & 0.14\\
    $\ell_1$ TV (AT, FGSM)                     & 22.06 & \num{8.8e-2} & 227 & 53.09  & \num{2.5e-3} &
    63.71 & 34.60 & 0.71 & 2.97 & 0.08 
    \\
    $\ell_2$ TV & 21.57 &\num{8.6e-2}& 253  & 53.77 & \num{2.6e-3} &
    61.72   & 44.81 & 0.73 & 2.86 & 0.08 \\
    $\ell_2$ TV + Lipschitz  & 21.73 & \textbf{\num{11.2e-2}} & \textbf{307}  & \textbf{46.79}  & \textbf{\num{3.2e-3}} &\textbf{55.92}   & \textbf{27.58} & \textbf{0.46} & \textbf{2.10} & \textbf{0.05} \\
  %  AT+WCAT \& $\tanh$ & 21.47 &\textbf{\num{13.6e-2}}&\textbf{382}  &    \textbf{42.58}  &\textbf{\num{4.0e-3}} & \textbf{48.98}    &\textbf{17.97}     & \textbf{0.35}   & \textbf{2.00} & \textbf{0.05} \\
    \hline
  \end{tabular}
}
  \label{tab:punch100}
\end{table*}

Here we present  adversarial robustness results for  models trained with the modified loss \eqref{main_equation}  on the CIFAR-10 and CIFAR-100 datasets
\cite{cifar_data}. We used the ResNeXt architecture \cite{resnext}. Models were trained with standard data augmentation and training
hyper-parameters for the CIFAR datasets, and cutout \cite{cutout}. 
%Unless otherwise specified, we set the Total Variation multiplier to $\varepsilon=0.01$ and the Lipschitz multiplier to $\lambda=0.1$.

Our  results are summarized in Tables \ref{tab:punch10} for CIFAR-10
and \ref{tab:punch100} for CIFAR-10. We implement four standard attacks using
the Foolbox library \cite{foolbox}: PGD in the 2-norm; Iterative FGSM; Boundary
attack; and the Carlini-Wagner attack. We found that when measured in the
2-norm, generally PGD was the most effective attack (on our reported models), whereas when measured in
the $\infty$-norm, Iterative FGSM was the most effective. See for example Figure~\ref{fig:attack-comparison}, where we plot error curves comparing the four attacks,
on both our best regularized model and an undefended model.  The error curve
measures the probability of successful attack given an adversarial distance. 
%Thus unless otherwise specified, all robustness results are derived from PGD in the 2-norm, and Iterative FGSM in the $\infty$-norm.

In Tables \ref{tab:punch10} and \ref{tab:punch100} we report test error on clean
images, and the error at adversarial perturbation $\varepsilon=0.1$ in the
2-norm (just visible to the human eye) and $\varepsilon=8/255$ in the
$\infty$-norm. We also report median adversarial distance (which corresponds to
the $x$-intercept of 50\% on the attack curves).
The median adversarial distance is a better overall characterization of adversarial robustness
than reporting values at one particular distance.

We also report  median time to
successful attack (as a ratio over one model call). The median time to attack is
a proxy for the \emph{cost} of a successful attacks. We make an analogy with
cryptography \cite{stallings2006cryptography}: there is no unbreakable code, but
there is a code which takes enormous effort to break.  Thus, as a secondary
goal, we can design models which are costly to attack.

\textbf{Comparison with other works.} On the CIFAR-10 dataset to our knowledge the current
state-of-the-art is \cite{madry_2017} (using adversarial training) for attacks
measured in the $\infty$-norm, and until recently was also the state-of-the-art
in the 2-norm. \cite{madry_2017} reports results at $\varepsilon = 8/255$ in the $\infty$-norm.  
Comparing with our training hyperparameters set to $\varepsilon = \lambda = 0.1$,  
our model achieved 91.6\% error, whereas Madry et al.~\cite{madry_2017} reports 54.2\% error. 
 In the 2-norm, our model achieves 37.2\% error at adversarial distance 100/255, whereas Madry et al.~\cite{madry_2017} reports roughly 90\% error at this distance. 
However, on clean images Madry et al.~\cite{madry_2017} reported 12.8\% test error,
whereas we obtain 5.4\% versus 4.1\% for the undefended model.
See Table~\ref{tab:compare_methods}. 

We also compare to the more recent work of Qian and Wegman \cite{qian2018} who report strong results 
for attacks measured in the 2-norm. %\citeauthor{qian2018}  enforce a small model Lipschitz constant with a hard weight penalty. 
Qian and Wegman \cite{qian2018} report results only at 2-norm of 1.5, which is a large value.
They obtain an error of  79.6\%, whereas we achieved around 90\% at this
distance. However, we note that our regularized model achieves roughly 5\% test
error, whereas Qian and Wegman \cite{qian2018} report nearly 23\% test error.  By comparison, at 23\% error, we had an adversarial robustness distance of 0.25.  Thus at some distance in between .25 and 1.5, the error curves cross, but this crossover point could not be obtained from their data.

\textbf{Defended models are still accurate.}
By choosing the hyper-parameters $\varepsilon$ and $\lambda$ (multipliers for
respectively the Total Variation and Lipschitz regularizers in \eqref{main_equation}) judiciously, we
are able to train robust models without significantly degrading test error.
For some architectures  we even observed improved test error.  For example,
on CIFAR-10, with ResNeXt-34 (2x32), the undefended
baseline model achieved 4.07\% misclassification error, whereas the regularized
model achieved 4.13\%. Another regularized model, which had nearly the same robustness, had better accuracy than the baseline, achieving 3.58\%  See Table~\ref{tab:punch10}.

On CIFAR-100, the undefended model achieved 21.2\% misclassification, while the
most heavily regularized model achieved 21.73\% error. 

Increasing the amount of regularization leads to worse test accuracy on clean
images, but improved adversarial robustness. For example in Figure
\ref{fig:stability-curves} we plot error curves for three models, with different
levels of regularization. The undefended network performs best on clean images,
but is not adversarially robust. Of the two regularized networks, the network
with strongest regularization ($\varepsilon=0.1$) has better robustness
properties than the network regularized with $\varepsilon=0.01$, with
respectively adversarial distances of 0.48 and 0.31. However, when measured by
test error, the ranking is reversed, with test error 5.37\% and 4.13\%
respectively. Thus a balance may struck depending on the desired outcome, with
test error degrading as the level of regularization is increased, consistent with \cite{tsipras2018robustness}. 

\textbf{Defended models are robust to attacks.}
Defended models have increased median adversarial distance and improved
adversarial robustness, as measured by adversarial test accuracy. Models with more
regularization leads to stronger adversarial defence, as shown in Tables
\ref{tab:punch10} and \ref{tab:punch100}.

\begin{figure*}
  \centering
  \subfigure[undefended model]{
    \includegraphics[height=2.75in]{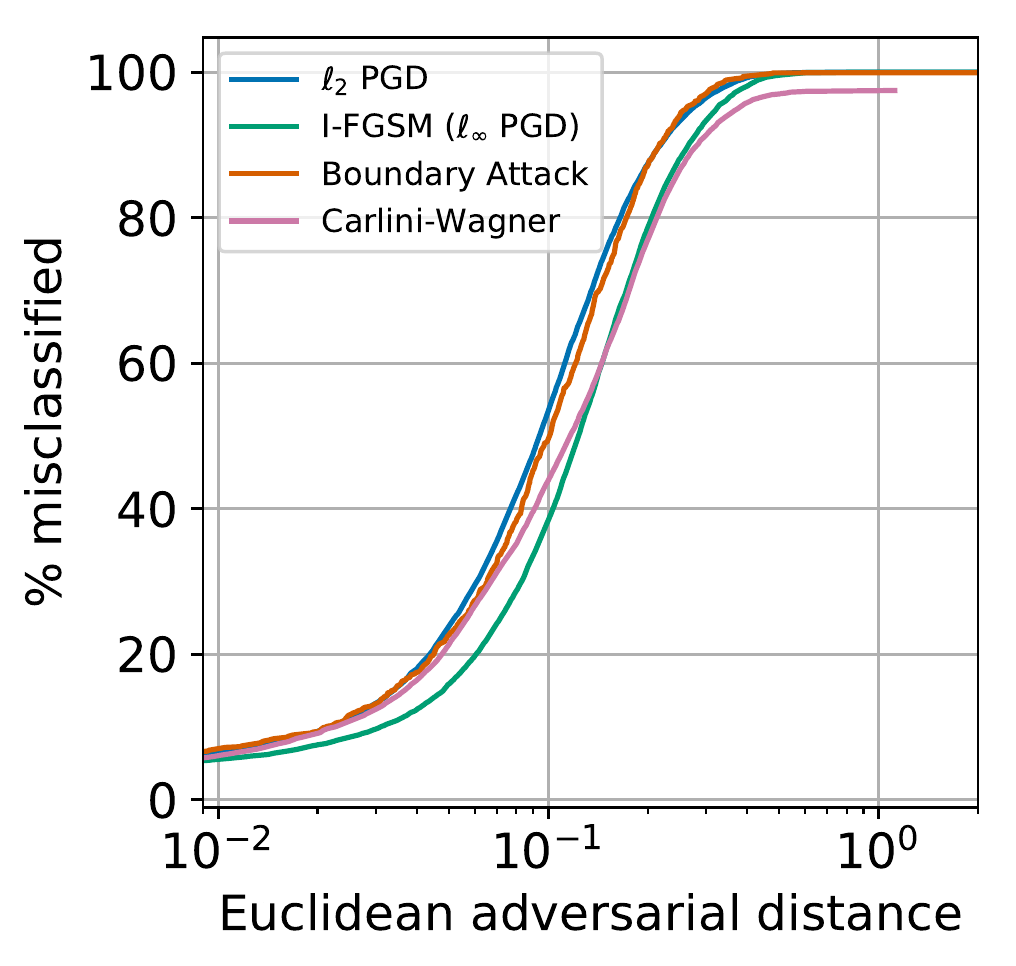}
  \label{fig:attack-l2}}
  \subfigure[best regularized model]{
    \includegraphics[height=2.75in]{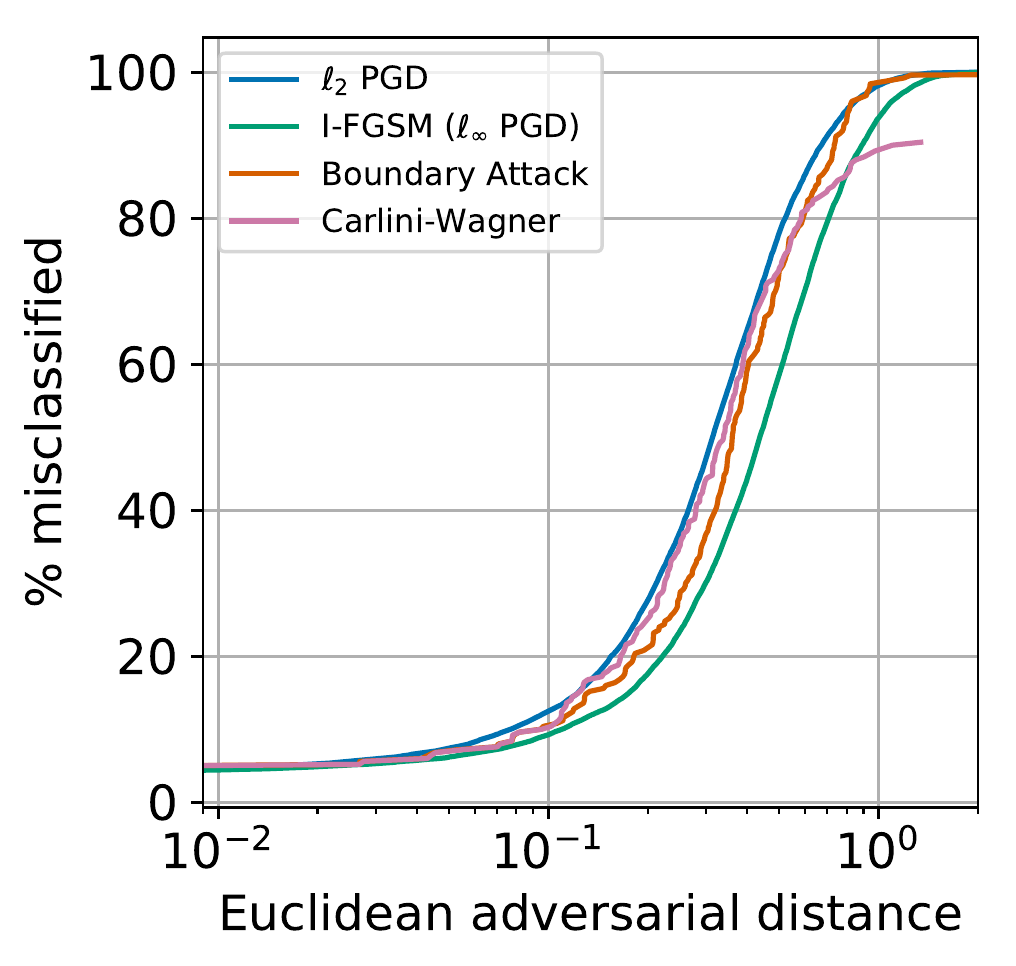}
  \label{fig:defended-l2}}
  \caption{Comparison of attack methods using error curves for
    ResNeXt-34 (2x32), on the CIFAR-10 test set. The error curve measures the
    probability of successful attack given an adversarial distance.
    \ref{fig:attack-l2} compares attacks on an undefended model;
  \ref{fig:defended-l2} compares attacks on our best regularized model.}
  \label{fig:attack-comparison}
\end{figure*}

\section{Robustness metrics}
Here we demonstrate that the measure of a model's robustness is given by the Total
Variation (average gradient norm) and the Lipschitz constant (maximum gradient)
norm of the model, or its loss.  This allows us to estimate improved robustness based on the improvement of the Total Variation and the Lipschitz constant of the models.   See Table~\ref{tab:relative}.

We estimated the Lipschitz constant of a model (or the loss of a model) using
\eqref{LipRad}.  This estimate  is based on sampling the norm of the gradient on
the test data.   This leads to an accurate and efficient method for estimating the Lipschitz
constant of the model, and the loss of a model. Our empirical estimate (see
Tables~\ref{tab:punch10} \& \ref{tab:punch100}), was usually close to one, and
was at most 14.  In contrast, by estimating the Lipschitz constant of a model as a
product of weight matrix norms, 
we obtained values between $10^{12}$ and $10^{23}$. 
We find that our estimate, \eqref{LipRad}, agrees
closely with the Lipschitz constant of the \emph{dataset}.
The Lipschitz constant of a dataset $\DD$ (in the $2,\infty$ norm) is the reciprocal of the minimum out-of-class $2$-norm distance.  Table~\ref{tab:lip-data} lists the Lipschitz constant of the training data for
common datasets. All datasets have Lipschitz constant close to or below one,
which is similar to the estimate of the Lipschitz constant of the model. 
\begin{table}
      \caption{Lipschitz constants in the ${2,\infty}$ norm, (equivalently $\min \|x_i - x_j\|$ for images in difference classes),  of training sets.  For 
CIFAR-100 we exclude duplicated images. }
  \centering
  %\resizebox{0.8\linewidth}{!}{
        \begin{tabular}{ l |l|l|l}
        \small MNIST &  {\small FashionMNIST} & {\small CIFAR-10} & {\small CIFAR-100}  \\
     \hline 
       0.417 & ~0.626 & ~0.364   & 1.245 \\
%     ${\infty,\infty}$ & 3.45 & 63.75 & 10.20   & 9.107 \\
%      \hline
    \end{tabular}
 %  }
    \label{tab:lip-data}
\end{table}

\begin{table}
\caption{Predicted and actual improvement of a robust regularized model over an
  undefended model for three measures of robustness. Prediction based on ratio of model Total Variation and model Lipschitz constant. Actual improvement based on ratio of median adversarial distance. }
  \centering
 % \resizebox{0.48\textwidth}{!}{
  \begin{tabular}{ l||   c |  c }
    Dataset & Predicted  & actual \\
    \hline
    CIFAR-10 & 3.4---6.4 & 3.4 \\
    CIFAR-100 & 2.8---4.1 & 2.4 \\
    \hline
  \end{tabular}
  %}
  \label{tab:relative}
\end{table}

We see that penalizing the Lipschitz constant of the \emph{loss} during model training reduces the
Lipschitz constant of the \emph{model} on the test data. For example, unregularized
models had Lipschitz constants of 13.7 and 1.89, on respectively CIFAR-10 and CIFAR-100. In
contrast, with Lipschitz regularization these constants reduced to to 2.14 and 0.46. Refer to Tables~\ref{tab:punch10} and~\ref{tab:punch100} for more details. 
This implies these models are more robust to
attack, which we observe reflected in the adversarial perturbation statistics.

A similar reduction is observed in the Total Variation of a model when models
are trained with Total Variation regularization. Tables \ref{tab:punch10} and
\ref{tab:punch100} demonstrate that the median
adversarial distance of a model is inversely proportional to the model's Total
Variation. We find that the improvement in median distance of a regularized model over an
undefended model corresponds to the relative improvement in the Total Variation
of the model.  See Table~\ref{tab:relative}.

\section{Attack detection} 
\begin{figure*}
  \centering
  \subfigure[Clean]{
    \includegraphics[height=1.4in]{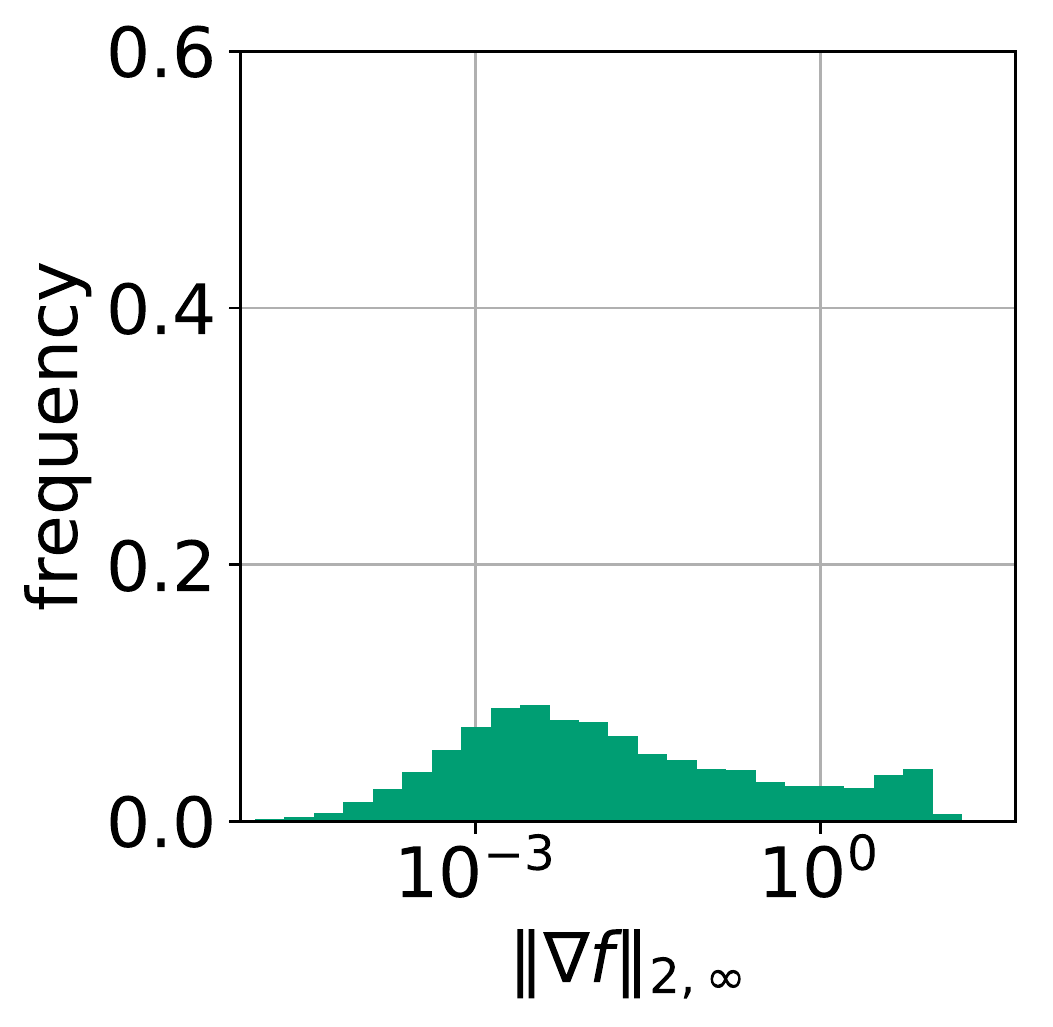}
  \label{fig:frq-clean}}
\subfigure[PGD attacked]{
  \includegraphics[height=1.4in]{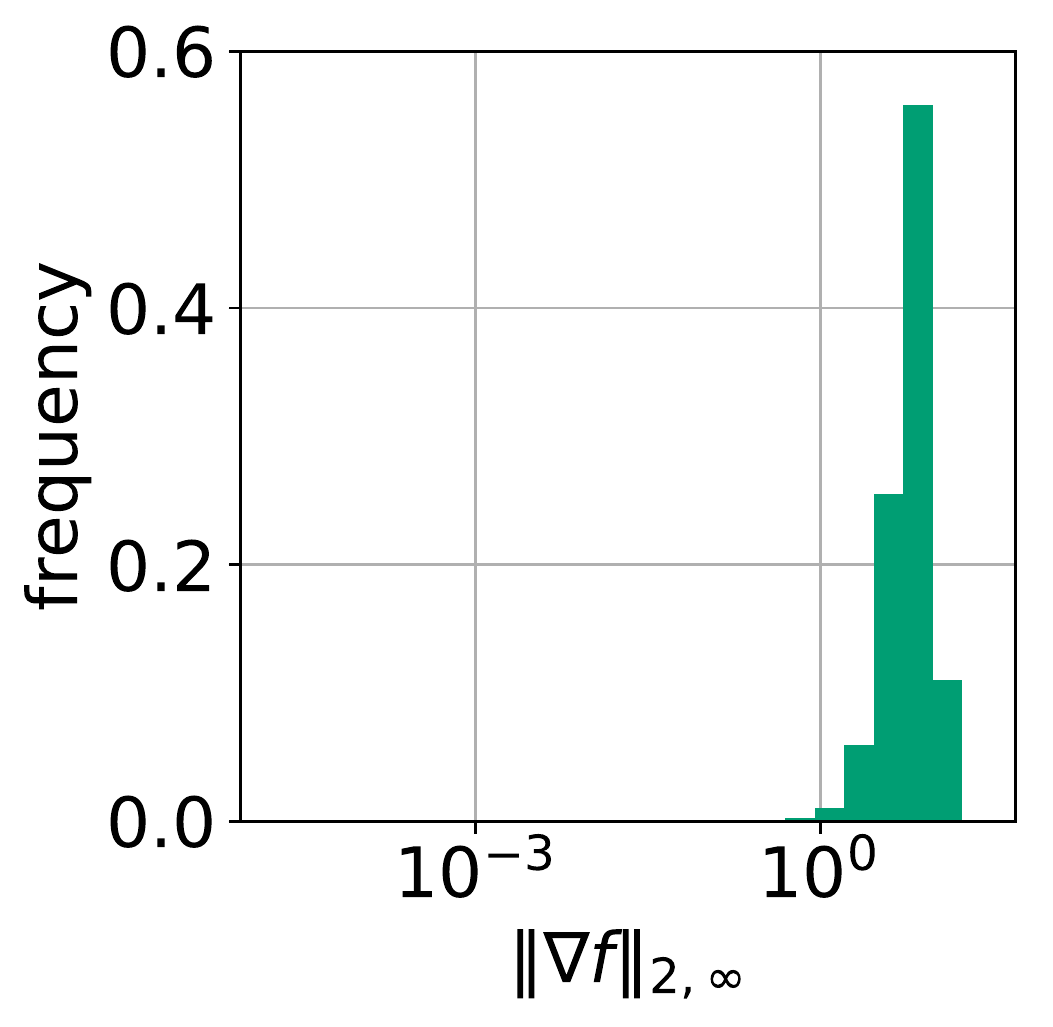}
  \label{fig:frq-pgd}}
\subfigure[Boundary attack]{
  \includegraphics[height=1.4in]{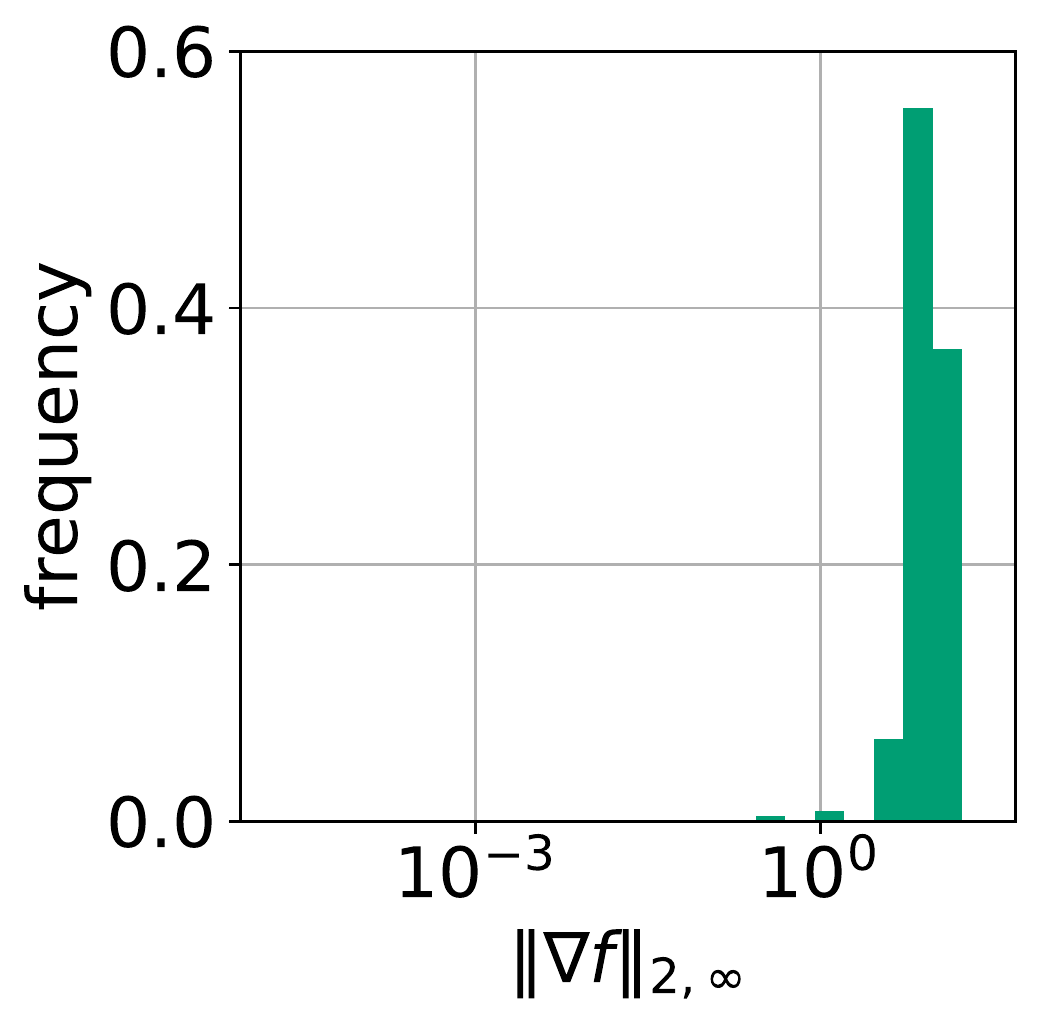}
  \label{fig:frq-bd}}
 \subfigure[evasive CW attack]{
  \includegraphics[height=1.4in]{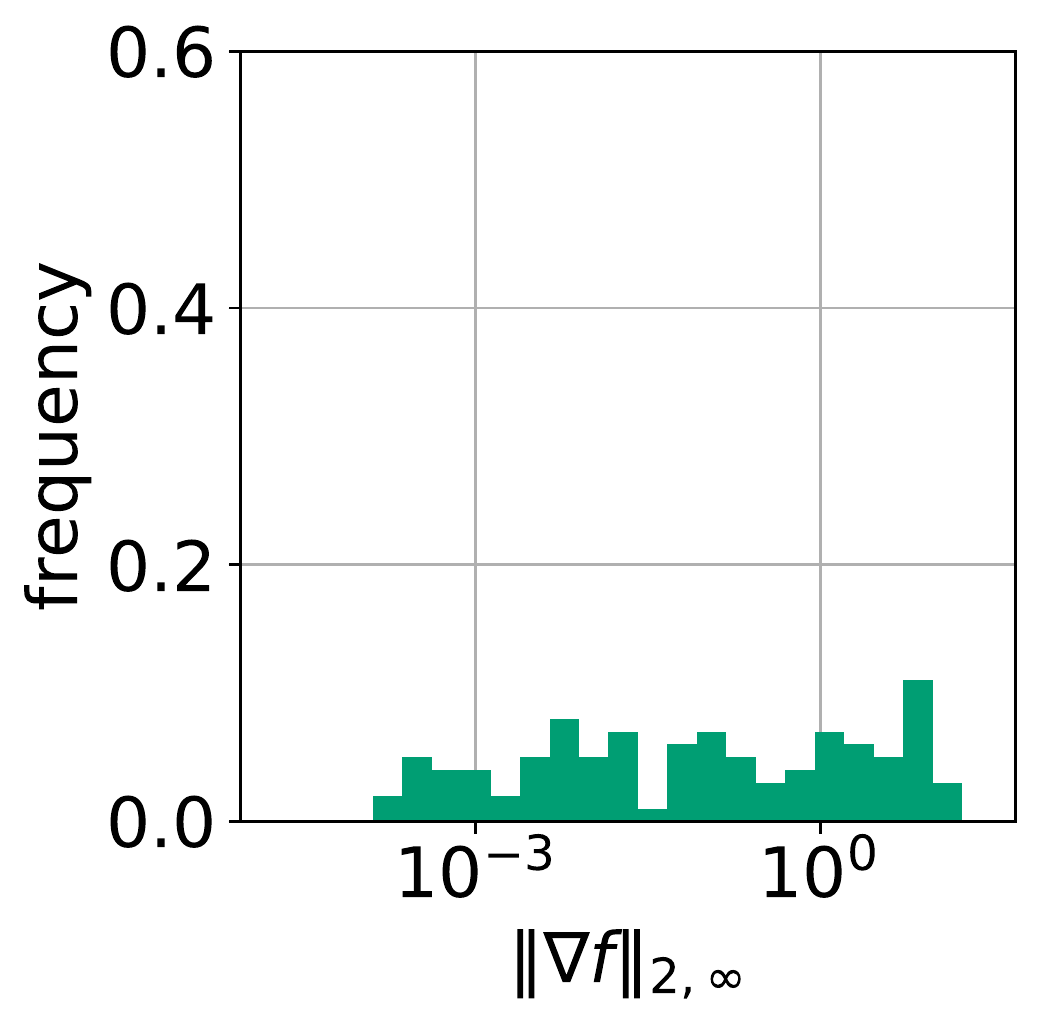}
  \label{fig:frq-lp}} 
  \caption{Frequency distribution of the norm of the model Jacobian
    $\norm{\nabla f(x)}_{2,\infty}$ on ResNeXt-34 (2x32)
    on CIFAR-10, using     
    \ref{fig:frq-clean}: Clean, \ref{fig:frq-pgd}: PGD attacked
    \ref{fig:frq-bd}: Boundary attacked, \ref{fig:frq-lp}: evasive-CW attacked 
    test images.
}
  \label{fig:detection}
\end{figure*}
In this section we empirically demonstrate that image vulnerability may also be
used to \emph{detect} adversarial examples.
We hypothesize that unless otherwise penalized, gradient based attacks will tend to move images to regions where the gradient of
the loss is large. Based on this heuristic, we propose the
norm of the loss gradient norm as criterion for detecting adversarial
perturbations.  Because the loss is not available during inference, we propose
using the norm of the model gradient as a rejection criteria:  an image has been
adversarially perturbed if
\begin{equation} 
  \| \nabla f(x)\|_{2,\infty} \geq c,
\end{equation}
for some threshold value, $c$.  The threshold is determined by setting the
significance level (the rate of false positives) to 5\%.   For example on
CIFAR-10 we obtained $c =2.45$ for our model.  The results are reported in
Table~\ref{tab:detection} and in Figure~\ref{fig:detection}. Only 6\% of clean
test images were rejected.  However, 100\% of Boundary attacks and Carlini-Wagner
attacks were detected, as well as 96\% of PGD attacked~images.   

This leads to the question, is it possible to successfully perturb all images in
the test set, and avoid detection?
We built a targeted attack, designed to avoid detection. Following the
recommendation of Athalye et al.~\cite{athalye2018obfuscated}, we use a Carlini-Wagner (CW) style
attack, modified with a penalty to avoid detection. We augmented the attack loss
function with a penalty for $\| \nabla \loss(x)\|^2_*$, which penalizes attacks
for being detectable. We call this attack an evasive Carlini-Wagner attack.
The evasive CW attack was successful at avoiding detection
78\% of the time, but in order to do so, it increased the median adversarial
distance significantly, from 0.31 to 0.81, see Table~\ref{tab:detection}.

\begin{table}
  \caption{Adversarial detection with ResNeXt-34 (2x32) on CIFAR-10. Clean
    images which the model correctly labels are perturbed until they are
    misclassified with four attack methods (PGD, Boundary attack,
    Carlini-Wagner, and an evasive Carlini-Wagner designed to avoid detection).
Images are rejected if $\norm{\nabla f(x)}_{2,\infty}>2.45$.}
  \centering
%\resizebox{\linewidth}{!}{
  \begin{tabular}{ p{0.27\linewidth}  |  r  r  r  r  r}    
    \multicolumn{1}{c}{}& \multicolumn{5}{c}{Image source} \\
    \cmidrule(l){2-6}
      &  clean & PGD & Boundary & CW & evasive CW\\ 
    \hline
    attack detected?  &	6\% &96\%  &100\% & 100\% & 22\% \\
    \hline
    median $\ell_2$  & -  & 0.31   & 0.36 & 0.34 & 0.81\\ 
    \hline
  \end{tabular}
%}
  \label{tab:detection}
\end{table}

%%%%%%%%%%%%%%%%%%%%%%%%%%%%
%\subsubsection*{Acknowledgments}
%The authors thank Bill Tubbs, Alex Iannantuono and Aram Pooladian for their assistance designing the experimental pipeline.
%The authors acknowledge the support of a Google gift which was used to support
%Bilal Abbasi during a collaboration at Google Brain Montr\'eal.
%Adam Oberman was partially supported by AFOSR grant FA9550-18-1-0167.
%%%%%%%%%%%%%%%%%%%%%%%%%%%%

%%%%%%%%%%%%%%%%%%%%%%%%%%%%

%{\footnotesize
%\bibliography{DNN}
%\bibliographystyle{alpha}
%}

\appendix

\section{Interpolation inequality}
\label{sec:interpolation}

We show here how to estimate the $L^\infty$ norm of a function by its $L^2$ norm and Lipschitz constant.
\begin{theorem}[Interpolation inequalities]\label{thm:interpolation}
Let $u:X \to Y$ be Lipschitz continuous with Lipschitz constant $L$. There exists $r>0$ depending only on $\M$ such that when $\|u\|_{L^\infty(\M,Y)}\leq Lr$ we have
\begin{equation}\label{eq:interpolation1}
\|u\|_{L^\infty(\M,Y)}^{m+2} \leq \frac{(m+2)(m+1)}{\alpha(m)}L^m\|u\|_{L^2(\M,Y)}^2,
\end{equation}
and when $\|u\|_{L^\infty(\M,Y)}\geq Lr$ we have
\begin{equation}\label{eq:interpolation2}
\|u\|^2_{L^\infty(\M,Y)}\leq \frac{2^{m+1}(m+1)}{\alpha(m)r^m}\|u\|_{L^2(\M,Y)}^2,
\end{equation}
where $\alpha(m)$ is the volume of the unit ball in $\R^m$.
\end{theorem}
\begin{proof}
We may assume, without loss of generality, that $Y=\R$. Let $x_0\in \M$ such that $|u|$ attains its max over $\M$ at $x_0$, and assume that $u(x_0)>0$.  Then for any $r>0$ sufficiently small, depending only on $\M$, the Riemannian exponential map $\exp_{x_0}:B(0,r)\subset T_{x_0}\M \to \M$ is a diffeomorphism between the ball $B(0,r)$ and the geodesic ball $B_\M(x_0,r)$, the Jacobian of $\exp_{x_0}$, denoted $J_{x_0}(v)$ satisfies $J_{x_0}(v)\geq \frac{1}{2}$ for all $|v|\leq r$, and the geodesic distance $d(x,x_0$ satisfies $d(x,x_0)\geq |x-x_0|$. Writing $K=\|u\|_{L^\infty(\M,Y)}=u(x_0)$, we we have
\begin{align*}
\|u\|^2_{L^2(\M,Y)}&\geq \int_{B_\M(x_0,r)} u(x)^2 \, dVol(x) \\
&\geq  \int_{B_\M(x_0,r)} (K - L|x-x_0|)_+^2 \, dVol(x) \\
&\geq \int_{B_\M(x_0,r)} (K - L d(x,x_0))_+^2 \, dVol(x) \\
&= \int_{B(0,r)\subset T_{x_0}\M} (K - L |v|)_+^2 J_{x_0}(v) \, dv, \\
&\geq \frac{1}{2}\int_{B(0,r)} (K - L |v|)_+^2 \, dv.
\end{align*}
If $r \geq K/L$ then we can compute
\[\int_{B(0,r)} (K - L |v|)_+^2 \, dv =\frac{2\alpha(m)}{(m+2)(m+1)}\frac{K^{m+2}}{L^m}.\]
This yields the estimate
\[\|u\|_{L^\infty(\M,Y)}^{m+2} \leq \frac{(m+2)(m+1)}{\alpha(m)}L^m\|u\|_{L^2(\M,Y)}^2.\]
If $r \leq K/L$ then we compute
\[\int_{B(0,r)} (K - L |v|)_+^2 \, dv\geq \frac{\alpha(m)r^m}{2^m(m+1)}K^2,\] 
which yields
\[\|u\|^2_{L^\infty(\M,Y)}\leq \frac{2^{m+1}(m+1)}{\alpha(m)r^m}\|u\|_{L^2(\M,Y)}^2,\]
which completes the proof.
\end{proof}

\section{Proof of Lemma~\ref{lem:discrepancy}}
\label{sec:app_proof}

We give here the proof of Lemma \ref{lem:discrepancy}.  A key tool in the proof is Bernstein's inequality \cite{boucheron2013concentration}, which we recall now for the reader's convenience. For $X_1,\dots,X_n$ \emph{i.i.d.}~with variance $\sigma^2 = \E[(X_i-\E[X_i])^2]$, if $|X_i|\leq M$ almost surely for all $i$ then Bernstein's inequality states that for any $\eps>0$
\begin{equation}\label{eq:bernstein}
\P\left( \left| \frac{1}{n}\sum_{i=1}^n X_i - \E[X_i] \right|> \eps\right)\leq 2\exp\left( -\frac{n\eps^2}{2\sigma^2 + 4M\eps/3} \right).
\end{equation}
\begin{proof}[Proof of Lemma \ref{lem:discrepancy}]
We note that it is sufficient to prove the result for $w\in H_L(X;Y)$ with $\int_\M w\rho \, dVol(x) = 0$. In this case, we have $w(x)=0$ for some $x\in \M$, and so $\|w\|_{L^\infty(X;Y)}\leq CL$. 

We first give the proof for $\M=X=[0,1]^m$.  We partition $X$ into hypercubes $B_1,\dots,B_N$ of side length $h>0$, where $N=h^{-m}$. Let $Z_j$ denote the number of $x_1,\dots,x_n$ falling in $B_j$. Then $Z_j$ is a Binomial random variable with parameters $n$ and $p_j = \int_{B_j}\rho \,dx \geq c h^m$. 
By the Bernstein inequality we have for each $j$ that
\begin{equation}\label{eq:bernapp}
\P\left( \left| \frac{1}{n}Z_j - \int_{B_j}\rho \,dx \right|> \eps\right)\leq 2\exp\left(-cnh^{-m}\eps^2 \right)
\end{equation}
provided $0 < \eps \leq h^m$. Therefore, we deduce
\begin{align*}
\frac{1}{n}\sum_{i=1}^n w(x_i) &\leq \frac{1}{n}\sum_{j=1}^N Z_j\max_{B_j} w\\
&\stackrel{\eqref{eq:bernapp}}{\leq}\sum_{j=1}^N \left( \int_{B_j}\rho \,dx + \eps \right)\max_{B_j} w\\
&\leq \sum_{j=1}^N \max_{B_j} w\int_{B_j}\rho \,dx + CLh^{-m}\eps\\
&\leq \sum_{j=1}^N (\min_{B_j} w + CLh)\int_{B_j}\rho \,dx + CLh^{-m}\eps\\
&\leq \sum_{j=1}^N \int_{B_j}w\rho \,dx + CLh^{-m}(h^{m+1} + \eps)\\
&=\int_X w \rho \, dx + CL(h + h^{-m}\eps)
\end{align*}
holds with probability at least $1-2h^{-m}\exp\left(-cnh^{-m}\eps^2 \right)$ for any $0 < \eps \leq h^{m}$. Choosing $\eps = h^{m+1}$ we have that
\begin{equation}\label{eq:upperbound}
\left|\frac{1}{n}\sum_{i=1}^n w(x_i) - \int_X w \rho \, dx\right|\leq CLh
\end{equation}
holds for all $u\in H_L(X;Y)$ with probability at least $1-2h^{-m}\exp\left(-cnh^{m+2} \right)$, provided $h \leq 1$. By selecting $nh^{m+2} = t\log(n)$
\[\sup_{w\in H_L(X;Y)}\left|\frac{1}{n}\sum_{i=1}^n w(x_i) - \int_\M w \rho \, dVol(x)\right|\leq CL\left( \frac{t\log(n)}{n} \right)^{\frac{1}{m+2}}\]
holds with probability at least $1-2t^{-\frac{m}{m+2}}n^{-(ct - 1)}$ for $t \leq n/\log(n)$. Since we have $\|w\|_{L^\infty(X;Y)}\leq CL$, the estimate
\[\sup_{w\in H_L(X;Y)}\left|\frac{1}{n}\sum_{i=1}^n w(x_i) - \int_\M w \rho \, dVol(x)\right|\leq CL,\]
trivially holds, and hence we can allow $t> n/\log(n)$ as well.

We sketch here how to prove the result on the manifold $\M$. We cover $\M$ with $k$ geodesic balls of radius $\eps>0$, denoted $B_\M(x_1,\eps),\dots,B_\M(x_k,\eps)$, and let $\varphi_1,\dots,\varphi_k$ be a partition of unity subordinate to this open covering of $\M$. For $\eps>0$ sufficiently small, the Riemannian exponential map $\exp_x:B(0,\eps)\subset T_x\M\to \M$ is a diffeomorphism between the ball $B(0,\eps) \subset T_x\M$ and the geodesic ball $B_\M(x,\eps)\subset \M$, where $T_x\M\cong \R^m$. Furthermore, the Jacobian of $\exp_x$ at $v\in B(0,r)\subset T_x\M$, denoted by $J_x(v)$, satisfies (by the Rauch Comparison Theorem)
\begin{equation}\label{eq:dist}
(1+C|v|^2)^{-1} \leq J_x(v) \leq 1 + C|v|^2.
\end{equation}
Therefore, we can run the argument above on the ball $B(0,r)\subset \R^m$ in the tangent space, lift the result to the geodesic ball $B_\M(x_i,\eps)$ via the Riemannian exponential map $\exp_x$, and apply the bound
\[\left|\frac{1}{n}\sum_{i=1}^n w(x_i) - \int_\M w \rho \, dVol(x)\right|\leq \sum_{j=1}^k\left|\frac{1}{n}\sum_{i=1}^n \varphi_j(x_i) w(x_i) - \int_{\M} \varphi_j w \rho \, dVol(x)\right|\]
to complete the proof.
\end{proof}

%\bibliography{DNN,ConvergentNN}
%\bibliographystyle{abbrv}

\end{document}